\documentclass{article}

\usepackage[preprint]{neurips_2024}
\usepackage[utf8]{inputenc} % allow utf-8 input
\usepackage[T1]{fontenc}    % use 8-bit T1 fonts
\RequirePackage[dvipsnames]{xcolor}      % Color for line numbers
\usepackage{color}

\definecolor{myred}{HTML}{F92F6C}
\definecolor{mygreen}{HTML}{16C232}
\definecolor{mypurple}{HTML}{AC6AEB}
\definecolor{myblue}{HTML}{2692F3}

\usepackage[pagebackref,breaklinks,colorlinks,citecolor=myblue,linkcolor=myred,urlcolor=mypurple]{hyperref}

\usepackage{url}            % simple URL typesetting
\usepackage{booktabs}       % professional-quality tables
\usepackage{amsfonts}       % blackboard math symbols
\usepackage{nicefrac}       % compact symbols for 1/2, etc.
\usepackage{microtype}      % microtypography

\usepackage[dvipsnames]{xcolor}
\usepackage{graphicx}
\usepackage{subcaption}
\usepackage{amsmath}
\usepackage{amssymb}
\usepackage{mathtools}
\usepackage{amsthm}
\usepackage{multirow}

\usepackage{enumitem}
\usepackage{placeins}
\usepackage{wrapfig}
\usepackage{bm}
\usepackage{bbm}
\usepackage{algorithm}
\usepackage{algorithmic}

\usepackage{listings}
\definecolor{code_blue}{HTML}{609AD1}
\definecolor{code_dkgreen}{rgb}{0,0.6,0}
\definecolor{code_gray}{rgb}{0.5,0.5,0.5}
\definecolor{code_mauve}{rgb}{0.58,0,0.82}
\theoremstyle{plain}
\newtheorem{theorem}{Theorem}[section]
\newtheorem{proposition}[theorem]{Proposition}

\theoremstyle{definition}

\theoremstyle{remark}

\title{\textit{Grokfast}: Accelerated Grokking by \\ Amplifying Slow Gradients}

\newcommand{\ece}{\spadesuit}
\newcommand{\ipai}{\heartsuit}

\author{%
  \href{https://jaerinlee.com}{Jaerin Lee}$^{*\ece}$ \And \href{https://scholar.google.com/citations?hl=en\&user=iuMRdnIAAAAJ}{Bong Gyun Kang}\thanks{Equal contribution.}$^{\,\,\,\ipai}$ \And \href{https://github.com/kihoon96}{Kihoon Kim}$^{\ipai}$ \And \href{https://cv.snu.ac.kr/index.php/kmlee}{Kyoung Mu Lee}$^{\ece\ipai}$\And\\[-2em]
  $^{\ece}$ASRI, Department of ECE, $^{\ipai}$Interdisciplinary Program in Artificial Intelligence,\\
  Seoul National University, Korea\\
  \texttt{\{ironjr,luckypanda,kihoon96,kyoungmu\}@snu.ac.kr}
}

\begin{document}

\maketitle

\begin{abstract}
One puzzling artifact in machine learning dubbed \emph{grokking} is where delayed generalization is achieved tenfolds of iterations after near perfect overfitting to the training data.
Focusing on the long delay itself on behalf of machine learning practitioners, our goal is to accelerate generalization of a model under grokking phenomenon.
By regarding a series of gradients of a parameter over training iterations as a random signal over time, we can spectrally decompose the parameter trajectories under gradient descent into two components: the fast-varying, overfitting-yielding component and the slow-varying, generalization-inducing component.
This analysis allows us to accelerate the grokking phenomenon more than $\times 50$ with only a few lines of code that amplifies the slow-varying components of gradients.
The experiments show that our algorithm applies to diverse tasks involving images, languages, and graphs, enabling practical availability of this peculiar artifact of sudden generalization.
Our code is available at \url{https://github.com/ironjr/grokfast}.
\end{abstract}

\section{Introduction}
\label{sec:1_intro}
%%%%%%%%%%%%%%%%%%%%%%%%%%%%%%%%%%%%%%%%%%%%%%%%%%%%%%%%%%%%%%%%%%%%%%%%%%%%%%%
\begin{wrapfigure}[14]{r}{0.4\linewidth}
    \centering
    \vspace{-1.5em}%
    \includegraphics[width=0.4\textwidth]{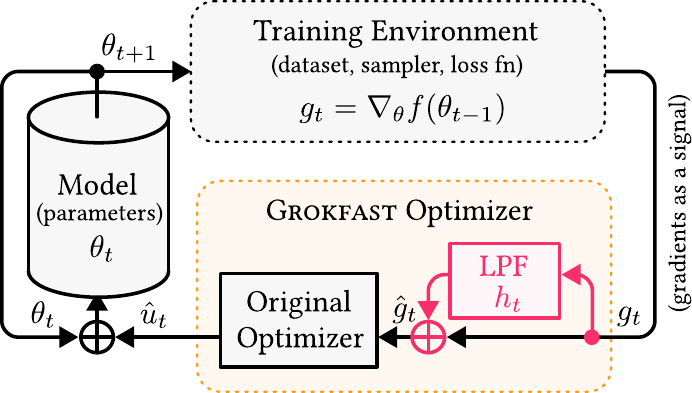}
    \vspace{-1.em}
    \caption{%
    \textsc{Grokfast} interprets training dynamics of a network parameter as a stochastic signal and amplify the low frequency variations for faster grokking.
    }
    \label{fig:concept}
\end{wrapfigure}
%%%%%%%%%%%%%%%%%%%%%%%%%%%%%%%%%%%%%%%%%%%%%%%%%%%%%%%%%%%%%%%%%%%%%%%%%%%%%%%
Grokking is a recently discovered phenomenon where generalization is achieved long after a model overfits to the training data.
The phenomenon was first reported by \cite{grok} for a two-layer Transformer~\citep{transformer} trained using a simple algorithmic dataset.
Later, \cite{omnigrok} have shown that similar artifacts are observed for various model architectures trained with a variety of datasets, including images, languages, and graphs.
Many theory-oriented works have tried to justify the effect by relating the grokking phenomenon to the previously known double descent phenomenon \citep{grok_dd_2,grok_dd_3}, yet its cause and sufficient conditions have not been fully characterized.

%%%%%%%%%%%%%%%%%%%%%%%%%%%%%%%%%%%%%%%%%%%%%%%%%%%%%%%%%%%%%%%%%%%%%%%%%%%%%%%
\begin{figure}[t]
\vskip 0.2in
\begin{center}
\centering
\subfloat[\small Accelerated grokking with \textsc{Grokfast}. \label{fig:figure_one:result}]{\includegraphics[width=0.50\linewidth]{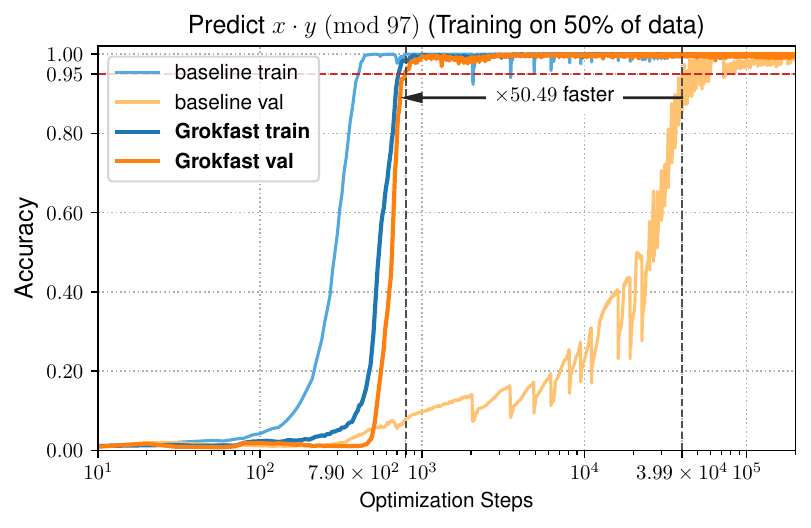}}
\hfill
\subfloat[\small Corresponding loss curves. \label{fig:figure_one:result_loss}]{\includegraphics[width=0.49\linewidth]{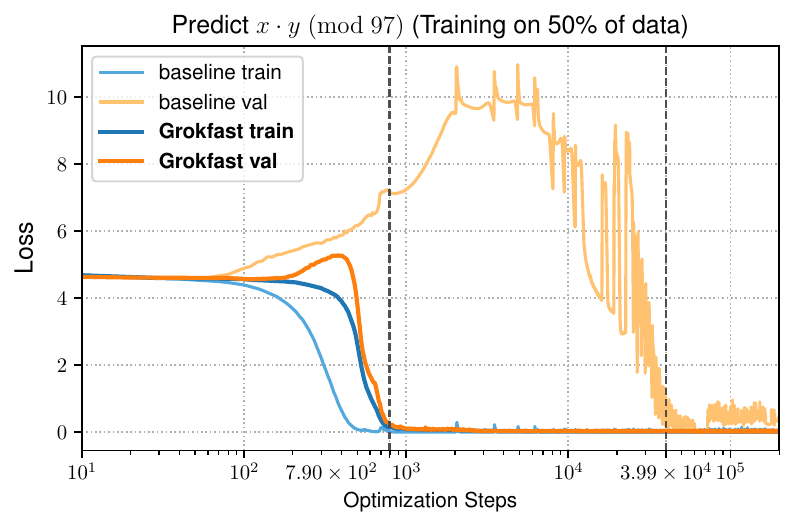}}
% \\[.2em]
\\
\caption{%
\textbf{Accelerating generalization of a model under grokking phenomenon.}
Our $\textsc{Grokfast}$ is a simple algorithmic modification upon existing optimizers to pull forward the time of event of sudden generalization after overfitting, also known as the \emph{grokking} phenomenon.
% Figure~\ref{fig:figure_one:result} is obtained by $\textsc{Grokfast-MA}$ with weight decay 0.01
}
\label{fig:figure_one}
\end{center}
\vskip -0.2in
\end{figure}
%%%%%%%%%%%%%%%%%%%%%%%%%%%%%%%%%%%%%%%%%%%%%%%%%%%%%%%%%%%%%%%%%%%%%%%%%%%%%%%

Apart from theoretical studies, this work takes a practitioner's standpoint to take advantage of the grokking phenomenon.
In previous reports on grokking \citep{grok,omnigrok}, generalization happens only after more than a tenfold of training iterations after overfitting.
This high demand of computational resources means less practical appeal to general machine learning practitioners who are often under dire resource constraints.
Achieving faster generalization in those overfitting systems is, therefore, a necessary step to fully explore the potential of this unusual behavior.
Our goal is, to this end, to accelerate the grokking phenomenon.

In the example training curve of a model under grokking in Figure~\ref{fig:figure_one}, the dynamics of the validation loss is few orders of magnitude slower than the dynamics of the training loss.
The change in the losses is a direct consequence of the change in the parameter values throughout the training session.
Hence, Figure~\ref{fig:figure_one} suggests that the parameter updates under grokking takes effect in two different timescales:
the fast-varying component of the parameter updates contributes to the rapid overfitting, and the slow-varying component contributes to the slow generalization.

% The speed of the parameter updates can be defined as follows:
We begin by treating the change of value $u(t)$ of each model parameter $\theta$ over training iteration $t$ from an optimizer as a discrete (random) signal over time.
As the optimizer iterates the training data, the value of each parameter $\theta(t)\,$, the loss $l(t)\,$, and its gradient $g(t) \coloneqq \partial l(t) / \partial \theta(t)$ drift with respect to the guidance from a sequence of randomly selected mini-batches sampled at each iteration $t\,$:
\begin{equation}
    \label{eq:1_intro:randproc_model}
    \theta(t + 1) = \theta(t) + u(g(t), t) = \theta({0}) + \sum_{\tau = 0}^{t} u(g(\tau), \tau)\,.
\end{equation}
The parameter update function $u(t) = u(g(t), t) = \theta(t + 1) - \theta(t)$ provides a simple abstraction of the underlying optimizer.
Different instances of iterative optimizers including SGD with various hyperparameters, e.g., learning rate and momentum, can be dealt with this notation.

Treating the optimization process as a collection of discrete random signals $u(t)$ allows us to consider its dual representation $U(\omega)$ in the \emph{frequency domain}.
Taking the discrete-time Fourier transform $\mathcal{F}$ of $u(t)$ with respect to the training iteration $t\,$, we obtain the spectral representation of the sequence of changes of a specific parameter $\theta\,$:
\begin{equation}
    \label{eq:1_intro:fourier}
    U(\omega) = \mathcal{F}\{u(t)\} =  \sum_{t = 0}^{T} u(t) e^{-i \omega t}\,,
\end{equation}
where $T$ is the total number of training iterations in this specific training session.
Under our assumption that the delayed generalization of grokking is a consequence of the slow-varying component of the parameter updates $u(t)\,$, the grokking phenomenon is directly related to the low-frequency part of the dual representation $U(\omega)\,$.
Further, for a first-order optimizer $u(g(t), t)\,$, an almost unanimous choice in deep learning applications, the gradients $g(t)$ are linearly correlated with the parameter updates $u(t)\,$.
Therefore, we can relate the low-frequency part of the gradient signal $G(\omega) = \sum_{t = 0}^{T} g(t) e^{-i \omega t}$ with the slow generalization under grokking.
Our hypothesis is that \emph{amplifying this low-frequency component of $G(\omega)$ accelerates the speed of generalization under the grokking phenomenon}.

In the following sections, we empirically demonstrate this hypothesis with a simple low-frequency gradient amplifier in various scenarios.
These include tasks involving various network architectures including Transformers~\citep{transformer}, MLPs, RNNs and (Graph-)ConvNets and diverse datasets such as algorithmic data, images, languages, and graphs that are treated to exhibit the grokking phenomenon~\citep{omnigrok}.
Our method is simple, taking only a few lines of additional code and is applicable to most of machine learning frameworks such as PyTorch \citep{pytorch}, with $\times 50$ faster exhibition of grokking as shown in Figure~\ref{fig:figure_one}.

\section{Amplifying the Low-Frequencies of the Stochastic Gradients}%Parameter Updates}
\label{sec:2_pre}
\subsection{Filter Design}
Amplifying the low-frequencies of the gradients $g(t)$ can be achieved by adding a low-pass filtered signal $g(t)$ to itself.
Let $h(t)$ be a discrete-time low-pass filter (LPF) defined over the training iteration $t\,$.
For simplicity, we assume a univariate time-invariant low-pass filter $h(t)$ uniformly applied across every model parameter $\theta\,$.
Using a convolution operator $*\,$, we denote the modified gradient $\hat{g}(t)$ as:
\begin{equation}
    \label{eq:2_pre:lpf_time}
    \hat{g}(t) = g(t) + h(t) * g(t)\,,
\end{equation}
which can then be plugged into the parameter update function $u$ of the optimizer:
\begin{equation}
    \label{eq:2_pre:update_lpf_time}
    \hat{u}(t) = u(\hat{g}(t), t) = u(g(t) + h(t) * g(t), t)\,.
\end{equation}
In the dual domain, equation~\eqref{eq:2_pre:lpf_time} is equivalent to:
\begin{align}
    \label{eq:2_pre:lpf_freq}
    \hat{G}(\omega) = G(\omega) + H(\omega) G(\omega) = (1 + H(\omega)) G(\omega)\,,
\end{align}
where $H(\omega) = \sum_{t = 0}^{T} h(t) e^{-i \omega t}$ is the transfer function of the filter $h(t)\,$.
Our goal can therefore be restated as to design a filter $h(t)$ with low-pass characteristics in its transfer function $H(\omega)\,$.

%%%%%%%%%%%%%%%%%%%%%%%%%%%%%%%%%%%%%%%%%%%%%%%%%%%%%%%%%%%%%%%%%%%%%%%%%%%%%%%
\begin{figure}[!t]
\begin{minipage}{.495\linewidth}
\begin{algorithm}[H]
\newcommand{\hlcolor}{OliveGreen}
   \caption{\textsc{Grokfast-MA}}
   \label{alg:avg}
\begin{algorithmic}[1]
   \STATE {\bfseries Param:} {\color{\hlcolor} window size $w\,$, scalar factor $\lambda\,$.}
   \STATE {\bfseries Input:} initial parameters $\theta_{0}\,$, stochastic objective function $f(\theta)\,$, optimizer's parameter update $u(g, t)$ from gradient $g$ at timestep $t\,$.
   \STATE {\bfseries begin:} $t \leftarrow 0\,$; {\color{\hlcolor} $Q \leftarrow \mathrm{Queue}(\mathrm{capacity}=w)$} %: a gradient queue.
   \WHILE {$\theta_{t}$ not converged}
       \STATE $t \leftarrow t + 1$
       \STATE $g_{t} \leftarrow \nabla_{\theta} f(\theta_{t - 1})\,$: Calculate gradients.
       \STATE {\color{\hlcolor} $\mathrm{Insert}(Q, g_{t})\,$: Insert gradients to $Q\,$.}
       \STATE {\color{\hlcolor} $\hat{g}_{t} \leftarrow g_{t} + \lambda \cdot  \mathrm{Avg}(Q)\,$: Filter gradients.}
       \STATE ${\color{\hlcolor} \hat{u}_{t}} \leftarrow u({\color{\hlcolor} \hat{g}_{t}}, t)\,$: Calculate update.
       \STATE $\theta_{t} \leftarrow \theta_{t - 1} + {\color{\hlcolor} \hat{u}_{t}}\,:$ Update parameters.
   \ENDWHILE
\end{algorithmic}
\end{algorithm}
\end{minipage}
\hfill
\begin{minipage}{.495\linewidth}
\begin{algorithm}[H]
\newcommand{\hlcolor}{OliveGreen}
    \caption{\textsc{Grokfast-EMA} (\textsc{Grokfast}).}
    \label{alg:main}
\begin{algorithmic}[1]
    \STATE {\bfseries Param:} {\color{\hlcolor} scalar momentum $\alpha\,$, factor $\lambda\,$.}
    \STATE {\bfseries Input:} initial parameters $\theta_{0}\,$, stochastic objective function $f(\theta)\,$, optimizer's parameter update $u(g, t)$ from gradient $g$ at timestep $t\,$.
    \STATE {\bfseries begin:} $t \leftarrow 0\,$; {\color{\hlcolor} $\mu \leftarrow \theta_{0}$: EMA of gradients.}
    \WHILE {$\theta_{t}$ not converged}
        \STATE $t \leftarrow t + 1$
        \STATE $g_{t} \leftarrow \nabla_{\theta} f(\theta_{t - 1})\,$: Calculate gradients.
        \STATE {\color{\hlcolor} $\mu \leftarrow \alpha \mu + (1 - \alpha) g_{t}\,$: Calculate EMA.}
        \STATE {\color{\hlcolor} $\hat{g}_{t} \leftarrow g_{t} + \lambda \mu\,$: Filter gradients.}
        \STATE ${\color{\hlcolor} \hat{u}_{t}} \leftarrow u({\color{\hlcolor} \hat{g}_{t}}, t)\,$: Calculate update.
        \STATE $\theta_{t} \leftarrow \theta_{t - 1} + {\color{\hlcolor} \hat{u}_{t}}\,:$ Update parameters.
    \ENDWHILE
\end{algorithmic}
\end{algorithm}
% \vspace{1em}
\end{minipage}
\end{figure}
\begin{figure}[t]
\newcommand{\figwidth}{0.245\linewidth}
\newcommand{\figheight}{0.163333\linewidth}
\begin{center}
\subfloat[\small Time plot of MA. \label{fig:filter:ma_time}]{\includegraphics[width=\figwidth,height=\figheight]{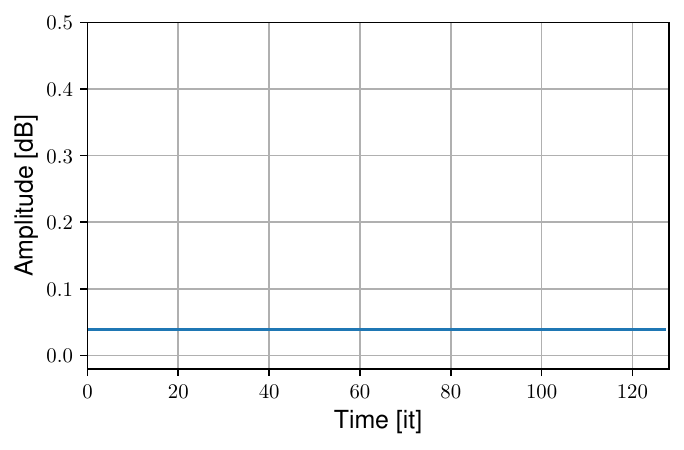}}
\hfill
\subfloat[\small Freq. plot of MA. \label{fig:filter:ma_freq}]{\includegraphics[width=\figwidth,height=\figheight]{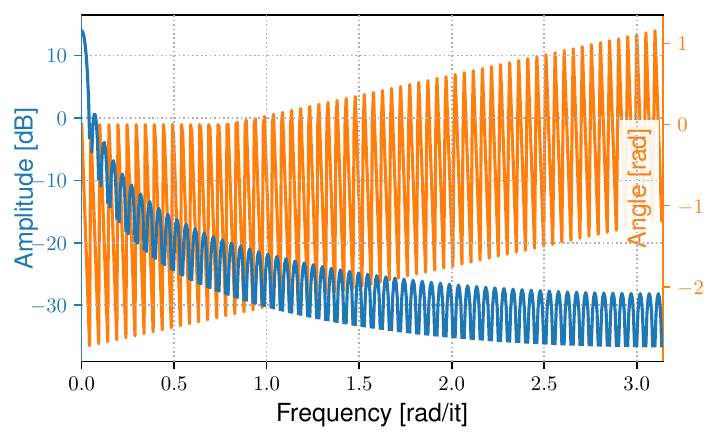}}
\hfill
\subfloat[\small Time plot of EMA. \label{fig:filter:ema_time}]{\includegraphics[width=\figwidth,height=\figheight]{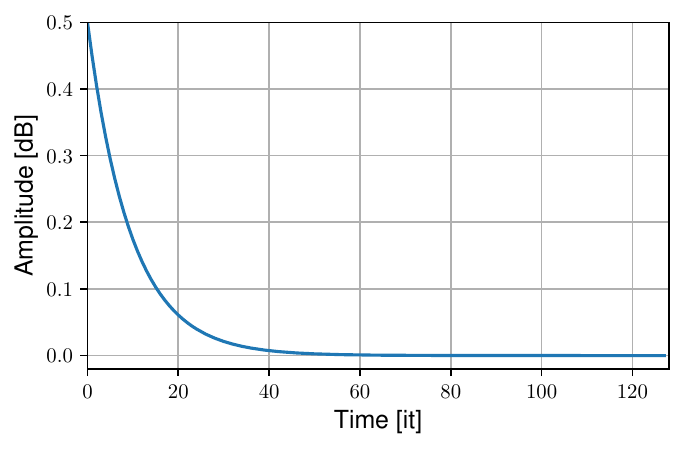}}
\hfill
\subfloat[\small Freq. plot of EMA. \label{fig:filter:ema_freq}]{\includegraphics[width=\figwidth,height=\figheight]{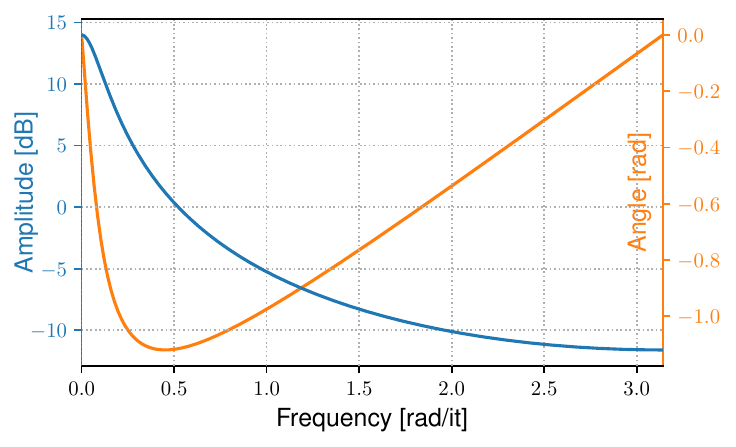}}
% \\[-.5em]
\caption{%
\textbf{Time and frequency domain plots of the gradient filters.}
Figures (a, b) and (c, d) depict the impulse responses and the transfer functions of the filters of Algorithm~\ref{alg:avg} and~\ref{alg:main}, i.e., the MA and the EMA filters $h(t)\,$.
We treat training iterations as discrete timesteps.
}
\label{fig:filter}
\end{center}
\vskip -0.2in
\end{figure}
%%%%%%%%%%%%%%%%%%%%%%%%%%%%%%%%%%%%%%%%%%%%%%%%%%%%%%%%%%%%%%%%%%%%%%%%%%%%%%%

For the demonstration of our initial claim, we first take the simplest strategy: the LPF $h(t)$ is a windowed moving average (MA) over a fixed-size window of width $w\,$.
\begin{equation}
    \label{eq:2_pre:lpf_avg}
    h(t) = \frac{\lambda}{w} \Pi \left( \frac{t}{w} - \frac{1}{2} \right)
    = \begin{cases}
        \lambda / w\,, & \text{if $0 \leq t < w$}\,.\\
        0\,, & \text{otherwise}.
    \end{cases}
\end{equation}
The function $\Pi$ stands for the Heaviside Pi function, which has value of one in the interval $[-0.5, 0.5]$ and of zero elsewhere.
This filter has only two scalar hyperparameters: the scalar factor $\lambda$ and the window size $w\,$.
As shown in Algorithm~\ref{alg:avg}, we implement this filter $h$ with a fixed-capacity queue $Q$ storing the intermediate parameter updates $u(t)$ into the queue $Q\,$.
The average of the queue $Q$ is the low-pass filtered gradients which is added to the current parameter update at each optimizer step.
%%%%%%%%%%%%%%%%%%%%%%%%%%%%%%%%%%%%%%%%%%%%%%%%%%%%%%%%%%%%%%%%%%%%%%%%%%%%%%%

%%%%%%%%%%%%%%%%%%%%%%%%%%%%%%%%%%%%%%%%%%%%%%%%%%%%%%%%%%%%%%%%%%%%%%%%%%%%%%%
\subsection{Experiment}
We first demonstrate our conjecture on the algorithmic dataset used in the first report on grokking \citep{grok}.
The network is a two-layer decoder-only Transformer~\citep{transformer} trained to predict the answer of a modular binary multiplication operation $x \cdot y\;(\mathrm{mod}\;97)\,$.
The training curve of this task is shown in Figure~\ref{fig:figure_one} as `baseline.'
Comparing the time to reach the accuracy of $0.95\,$, the generalization, i.e., the late saturation of the validation accuracy, happens after $\times 97.3$ iterations after the rapid saturation of the training accuracy (the overfitting).
Figure~\ref{fig:ma} shows empirical proof of effectiveness of Algorithm~\ref{alg:avg}, our $\textsc{Grokfast-MA}$ algorithm on this task.
Choosing the hyperparameters from a simple grid search over $\lambda \in \{1, 2, 5, 10\}$ and $w \in \{2, 5, 10, 20, 50, 100, 200\}\,$, we found that the filter works best when $\lambda = 5$ and $w = 100\,$.
As shown in Figure~\ref{fig:ma}, the number of iterations taken to make the validation accuracy reach 95\% of the training accuracy is reduced by $\times 13.57\,$, which is a remarkable reduction of training iterations.

%%%%%%%%%%%%%%%%%%%%%%%%%%%%%%%%%%%%%%%%%%%%%%%%%%%%%%%%%%%%%%%%%%%%%%%%%%%%%%%
\begin{figure}[t]
\newcommand{\figwidth}{0.495\linewidth}
\begin{center}
\subfloat[\small Accuracy with respect to amplifier gain $\lambda\,$. \label{fig:ma:acc_gain}]{\includegraphics[width=\figwidth]{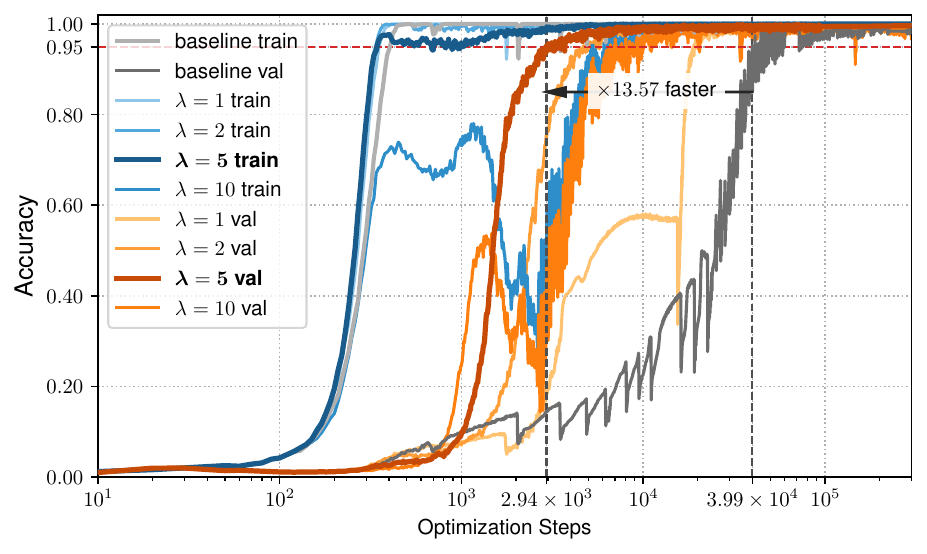}}
\hfill
\subfloat[\small Loss with respect to amplifier gain $\lambda\,$. \label{fig:ma:loss_gain}]{\includegraphics[width=\figwidth]{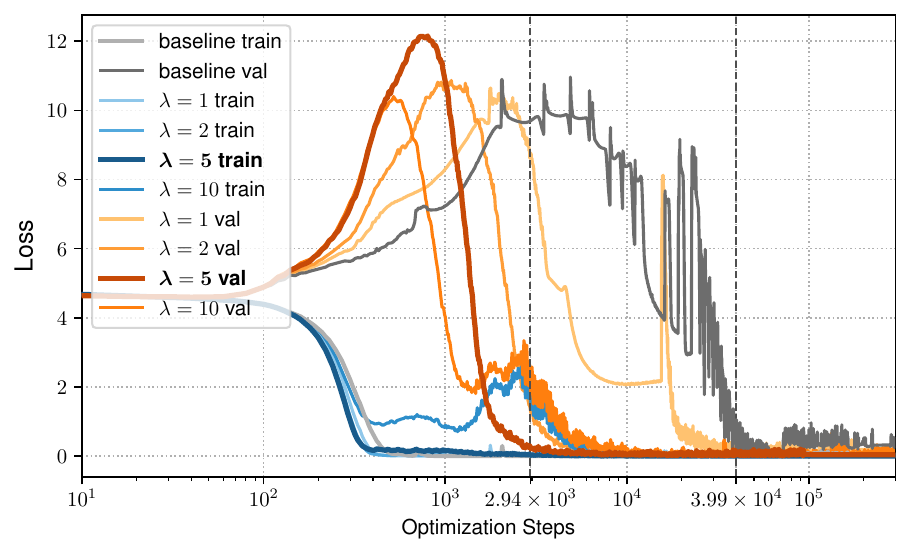}}
\\
\subfloat[\small Accuracy with respect to window size $w\,$. \label{fig:ma:acc_window}]{\includegraphics[width=\figwidth]{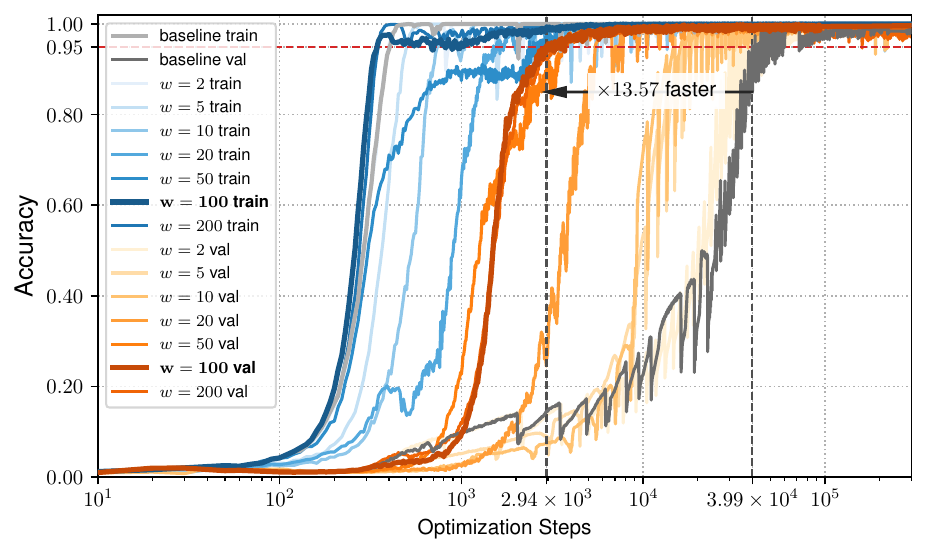}}
\hfill
\subfloat[\small Loss with respect to window size $w\,$. \label{fig:ma:loss_window}]{\includegraphics[width=\figwidth]{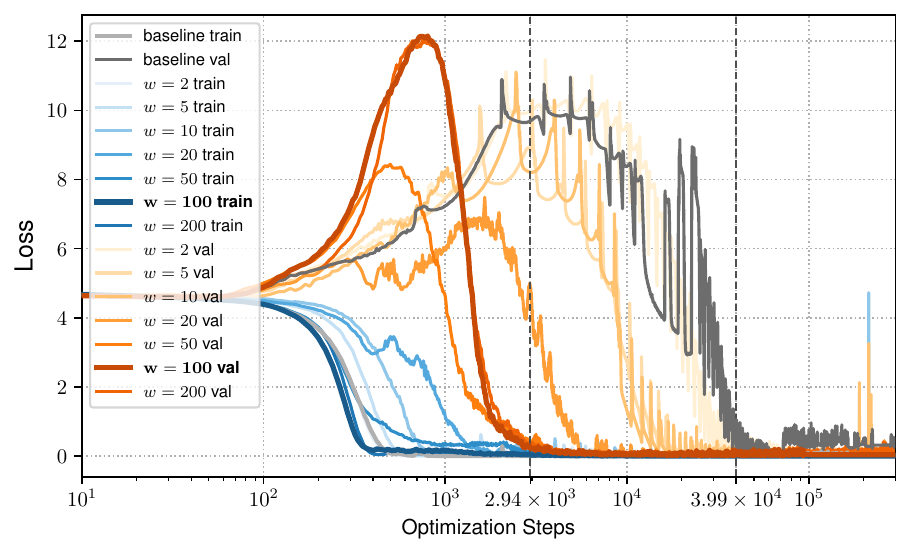}}
\caption{%
\textbf{Acceleration of delayed generation with \textsc{Grokfast-MA}.}
The amount of acceleration relies on the two hyperparameters, amplifier gain $\lambda$ and window size $w\,$.
Each hyperparameter has a sweet spot; increasing one arbitrarily does not guarantee faster acceleration.
Figures~\ref{fig:ma:acc_gain} and~\ref{fig:ma:loss_gain} use $w = 100$ except the baseline.
Figures~\ref{fig:ma:acc_window} and~\ref{fig:ma:loss_window} use $\lambda = 5$ except the baseline.
}
\label{fig:ma}
\end{center}
\vskip -0.2in
\end{figure}
%%%%%%%%%%%%%%%%%%%%%%%%%%%%%%%%%%%%%%%%%%%%%%%%%%%%%%%%%%%%%%%%%%%%%%%%%%%%%%%
\subsection{Discussion}
\label{sec:2_pre:discussion}
%%%%%%%%%%%%%%%%%%%%%%%%%%%%%%%%%%%%%%%%%%%%%%%%%%%%%%%%%%%%%%%%%%%%%%%%%%%%%%%
\begin{wrapfigure}[12]{r}{0.4\linewidth}
    \centering
    \vspace{-4.5em}%
    \includegraphics[width=0.4\textwidth]{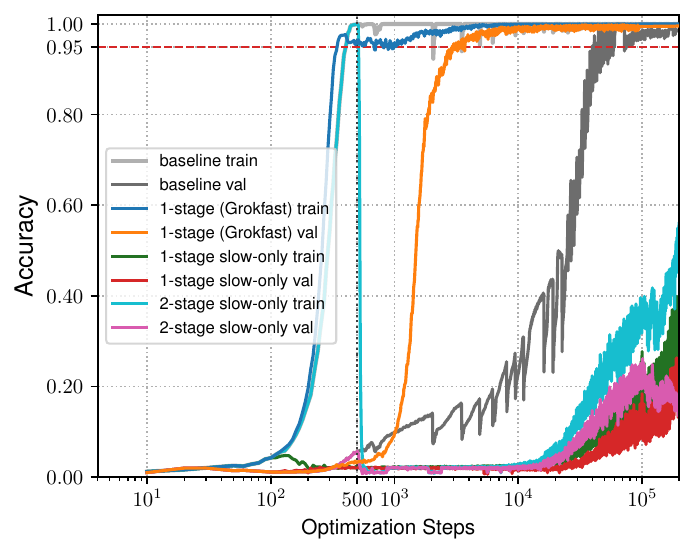}
    \vspace{-1.5em}
    \caption{%
    Although \textit{adding} the slow component of the gradients is effective in accelerating grokking, the slow component \textit{cannot} be used alone as a replacement. 
    }
    \label{fig:discussion_ma:slowonly}
\end{wrapfigure}
%%%%%%%%%%%%%%%%%%%%%%%%%%%%%%%%%%%%%%%%%%%%%%%%%%%%%%%%%%%%%%%%%%%%%%%%%%%%%%%
Figure~\ref{fig:ma} demonstrates high effectiveness of our approach.
However, few questions are still left unanswered regarding the modified training dynamics and the combined effect with the \emph{weight decay}, which is previously shown to be another important algorithmic factor that governs the grokking effect~\citep{omnigrok}.
We devise more experiments to answer these questions:

%%%%%%%%%%%%%%%%%%%%%%%%%%%%%%%%%%%%%%%%%%%%%%%%%%%%%%%%%%%%%%%%%%%%%%%%%%%%%%%
\paragraph{Q1. Are both slow and fast gradients necessary?}
Our approach is based on our belief that the low-pass filtered gradient updates, or the \emph{slow} gradients, contribute to the generalization.
The most obvious question is then: can we \textit{not} use the \emph{fast} gradients and replace the original sequence of gradients with the low-pass filtered components?
Using only the slow gradients calculated from a moving average filter in Algorithm~\ref{alg:avg} is equivalent to using larger, overlapping minibatches.
We conduct an experiment with a modified algorithm that replaces the line 8 of Algorithm~\ref{alg:avg} with $\hat{g}_{t} \leftarrow \lambda \cdot \mathrm{Avg}(Q)\,$.
Figure~\ref{fig:discussion_ma:slowonly} shows the result.
\textit{1-stage} means that the gradient replacement happens from the beginning of the training, which is set by default, and \textit{2-stage} means that the effect of \textsc{Grokfast} happens after the model overfits to the training data at iteration $500\,$.
The results clearly reveal that removing the original gradients leads to much slower and unstable training.
In conjunction with the result in Figure~\ref{fig:ma}, we can conclude that both the fast and the slow components of the gradients are necessary for faster grokking.
% that our $\textsc{Grokfast-MA}$ algorithm leads to faster generalization than the baseline that only utilizes the \textit{fast} gradients

%%%%%%%%%%%%%%%%%%%%%%%%%%%%%%%%%%%%%%%%%%%%%%%%%%%%%%%%%%%%%%%%%%%%%%%%%%%%%%%
\paragraph{Q2. Exploiting state transition in the training of a model under grokking.}
%%%%%%%%%%%%%%%%%%%%%%%%%%%%%%%%%%%%%%%%%%%%%%%%%%%%%%%%%%%%%%%%%%%%%%%%%%%%%%%
\begin{wrapfigure}[20]{l}{0.4\linewidth}
    \centering
    \vspace{-1.5em}%
    \includegraphics[width=0.4\textwidth]{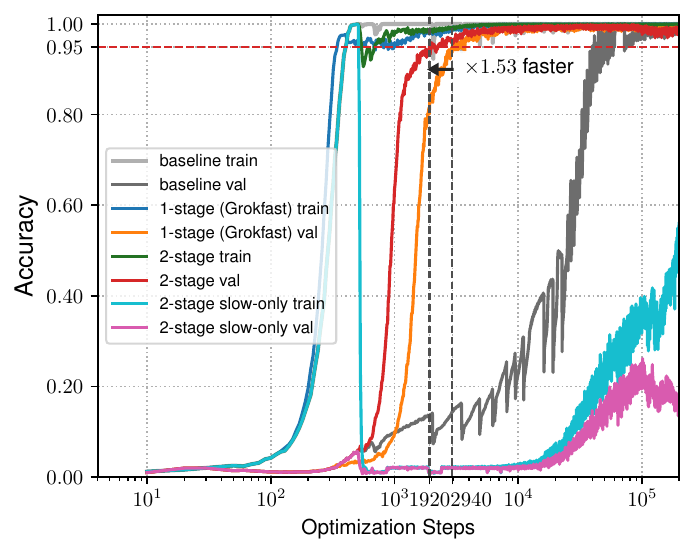}
    \vspace{-1.5em}
    \caption{%
    We can further accelerate the grokking effect with a two-staged algorithm, by applying $\textsc{Grokfast-MA}$ \textit{after} the model is overfitted (after 500 $\mathrm{its}$).
    }
    \label{fig:discussion_ma:staged}
    \captionof{table}{%
    Summary of results of Figure~\ref{fig:discussion_ma:staged}.
    }
    \label{tab:discussion_ma:staged}
    \resizebox{0.4\textwidth}{!}{%
        \begin{tabular}{llll}%
        \toprule
        Name & $\hat{g}_{t}$ at ($\mathrm{A} \rightarrow \mathrm{B}$) & $\hat{g}_{t}$ at ($\mathrm{B} \rightarrow \mathrm{C}$) & Iterations at $\mathrm{acc} \geq 0.95$ \\
        \midrule
        Baseline & ${g}_{t}$ & ${g}_{t}$ & 39,890 [$\mathrm{its}$] ($\times 1$) \\
        1-Stage & ${g}_{t} + \lambda \cdot \mathrm{Avg}(Q)$ & ${g}_{t} + \lambda \cdot \mathrm{Avg}(Q)$ & 2,940 [$\mathrm{its}$] ($\times 13.57$) \\
        \textbf{2-Stage} & ${g}_{t}$ & ${g}_{t} + \lambda \cdot \mathrm{Avg}(Q)$ & \textbf{1,920 [$\mathrm{its}$] ($\times 20.78$)} \\
        2-Stage Slow-only & ${g}_{t}$ & $\lambda \cdot \mathrm{Avg}(Q)$ & Not converged \\
        \bottomrule
    \end{tabular}%
    }
\end{wrapfigure}
%%%%%%%%%%%%%%%%%%%%%%%%%%%%%%%%%%%%%%%%%%%%%%%%%%%%%%%%%%%%%%%%%%%%%%%%%%%%%%%
We can alternatively interpret the training dynamics of a model under the grokking phenomenon as a \textit{state transition}.
In this viewpoint, the model sequentially goes through three distinct stages:
($\mathrm{A}$) \textit{initialized}, where both training and validation losses are not saturated,
($\mathrm{B}$) \textit{overfitted}, where the training loss is fully saturated but the validation loss is not, and
($\mathrm{C}$) \textit{generalized}, where both losses are fully saturated.
In the experimental setting of Figure~\ref{fig:ma}, state transition of $\mathrm{A} \rightarrow \mathrm{B}$ happens roughly after iteration $500\,$.
This interpretation allows us to try out a staged strategy for optimization, where different algorithms are applied to the model during the two transition phases $\mathrm{A} \rightarrow \mathrm{B}$ (from iteration 0 to 499) and $\mathrm{B} \rightarrow \mathrm{C}$ (after iteration 500) as described in Table~\ref{tab:discussion_ma:staged}.
Figure~\ref{fig:discussion_ma:staged} and Table~\ref{tab:discussion_ma:staged} summarize the result of the experiment.
As the results show, we can accelerate the grokking effect further by $\times 1.53$ by separating the training stage of the model and applying $\textsc{Grokfast-MA}$ only after the model becomes overfitted, suggesting an adaptive optimizer.

%%%%%%%%%%%%%%%%%%%%%%%%%%%%%%%%%%%%%%%%%%%%%%%%%%%%%%%%%%%%%%%%%%%%%%%%%%%%%%%
\paragraph{Q3. Synergistic effect with weight decay.}
%%%%%%%%%%%%%%%%%%%%%%%%%%%%%%%%%%%%%%%%%%%%%%%%%%%%%%%%%%%%%%%%%%%%%%%%%%%%%%%
\begin{wrapfigure}[17]{r}{0.56\linewidth}
    \centering
    \vspace{-1.5em}%
    \includegraphics[width=0.56\textwidth]{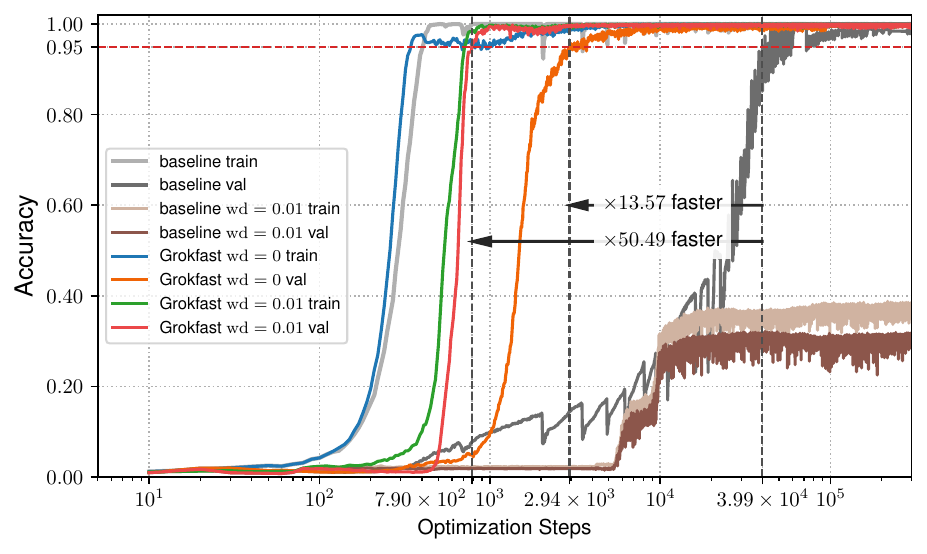}
    \vspace{-1.5em}
    \caption{%
    The acceleration effect of $\textsc{Grokfast-MA}$ is greatly enhanced when accompanied with appropriate value of weight decay.
    However, the weight decay alone not always yield beneficial results.
    }
    \label{fig:discussion_ma:wd}
\end{wrapfigure}
%%%%%%%%%%%%%%%%%%%%%%%%%%%%%%%%%%%%%%%%%%%%%%%%%%%%%%%%%%%%%%%%%%%%%%%%%%%%%%%
Besides from our gradient filtering approach, the authors of Omnigrok~\citep{omnigrok} have suggested that the weight decay hyperparameter is a critical determinant of the grokking phenomenon.
According to the report, the grokking phenomenon appears and even becomes faster when the weight decay becomes larger.
We, therefore, conduct additional experiments to find out how these two approaches affect the model when applied together.
The results are summarized in Figure~\ref{fig:discussion_ma:wd}.
Compared with the result from $\textsc{Grokfast-MA}$ with no weight decay ({\color[HTML]{f06406} orange}), applying the weight decay ({\color[HTML]{d62728} red}) generally yields even faster generalization.
The maximum acceleration appears at $\mathrm{wd} = 0.01$ with $\times 3.72$ faster generalization than $\textsc{Grokfast-MA}$ with no weight decay.
We choose this result of $\times 50.49$ faster grokking to be our main demonstration in Figure~\ref{fig:figure_one:result}.
Interestingly, Figure~\ref{fig:discussion_ma:wd} also reveals that applying the same weight decay with no $\textsc{Grokfast-MA}$ ({\color[HTML]{8c564b} brown}) makes the training unstable.
The results demonstrates that applying our gradient filtering and setting up a proper weight decay together gives synergistic benefits.

%%%%%%%%%%%%%%%%%%%%%%%%%%%%%%%%%%%%%%%%%%%%%%%%%%%%%%%%%%%%%%%%%%%%%%%%%%%%%%%

\subsection{Limitations}
\label{sec:2_pre:limit}
Although Algorithm~\ref{alg:avg} shows highly effective results, it requires $w$ times more memory to store all the previous gradients, limiting its utilization.
Replication of the model parameters also makes the training slower; using $w = 100\,$, the training time per iteration is increased by $\times 2.4\,$ measured with a single 1080 Ti GPU.
Still, the reduction of wall clock time before the delayed generalization of the results in Figure~\ref{fig:discussion_ma:wd} is $\times 20.5\,$, which is also a notable reduction of time.
Though the computation time does not linearly scale with the memory requirements, Algorithm~\ref{alg:avg} is not generally applicable to the larger models.
This problem leads to our new design utilizing sequential filter in the next section. % of Algorithm~\ref{alg:main}.

\section{Grokfast with Exponential Moving Average Gradient Filter}
\label{sec:3_method}
In the previous section, we empirically prove that using an LPF to the sequence of model parameter updates leads to faster generalization under the grokking phenomenon.
However, for practical purposes, we require an LPF design with a smaller memory footprint.
To this end, we modify Algorithm~\ref{alg:avg} by replacing the windowed moving average with an exponential moving average (EMA) filter.
The impulse response of the filter becomes:
\begin{equation}
    \label{eq:2_pre:lpf_ema}
    h(t) = \lambda (1 - \alpha) \sum_{\tau = 0}^{t} \alpha^{\tau} \delta (t - \tau)
    = \lambda \alpha^{t} (1 - \alpha)\,,
\end{equation}
where $\delta (t)$ is the discrete unit impulse at the origin.
This filter also has two hyperparameters: the scalar factor $\lambda$ and the scalar momentum $\alpha\,$.
The corresponding Algorithm~\ref{alg:main} only requires additional memory with the same size of the model itself, reducing $\times 50$ amount of required memory compared to Algorithm~\ref{alg:avg}.
The time and the frequency responses of the filters are compared in Figure~\ref{fig:filter}.

\section{Experiment}
\label{sec:4_exp}
%%%%%%%%%%%%%%%%%%%%%%%%%%%%%%%%%%%%%%%%%%%%%%%%%%%%%%%%%%%%%%%%%%%%%%%%%%%%%%%
\begin{figure}[t]
\newcommand{\figwidth}{0.495\linewidth}
\newcommand{\figwidtha}{0.332\linewidth}
\begin{center}
\subfloat[\small Accuracy of the modular multiplication task. \label{fig:ema:acc}]{\includegraphics[width=\figwidth]{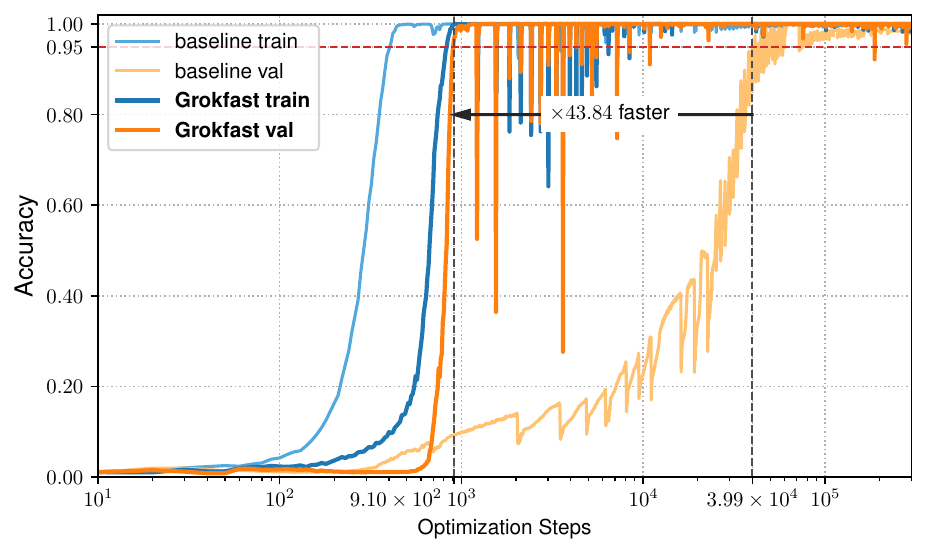}}
\hfill
\subfloat[\small Loss of the modular multiplication task. \label{fig:ema:loss}]{\includegraphics[width=\figwidth]{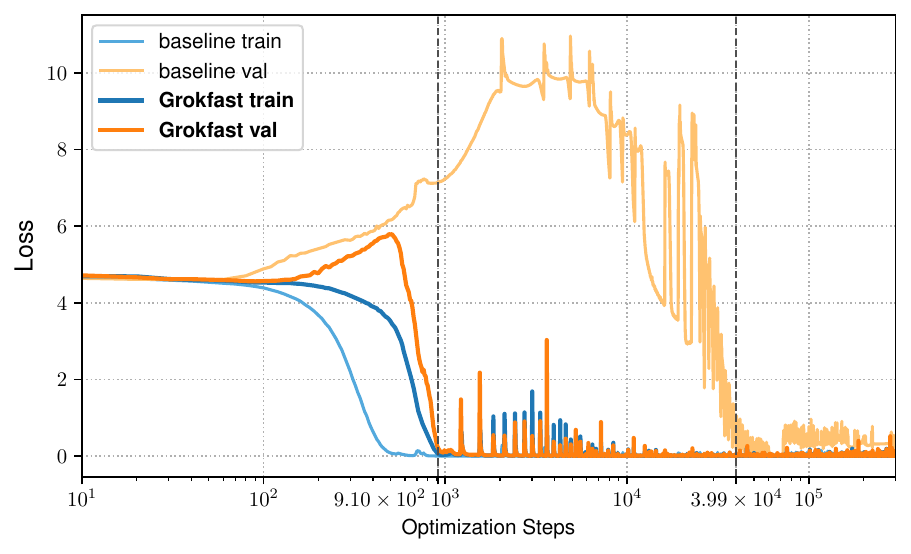}}
\\
\subfloat[\small Accuracy w.r.t. amplifier gain $\lambda\,$. \label{fig:ema:acc_gain}]{\includegraphics[width=\figwidtha]{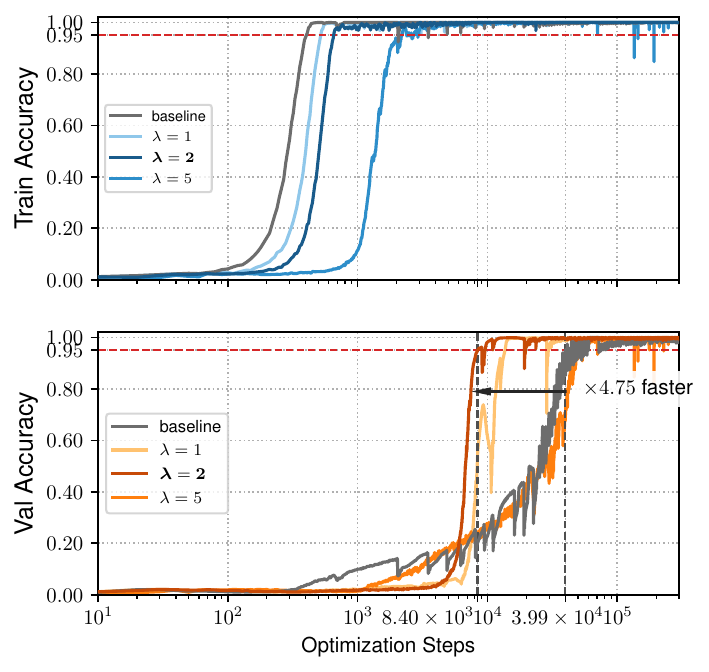}}
\hfill
\subfloat[\small Accuracy w.r.t. momentum $\alpha\,$. \label{fig:ema:acc_window}]{\includegraphics[width=\figwidtha]{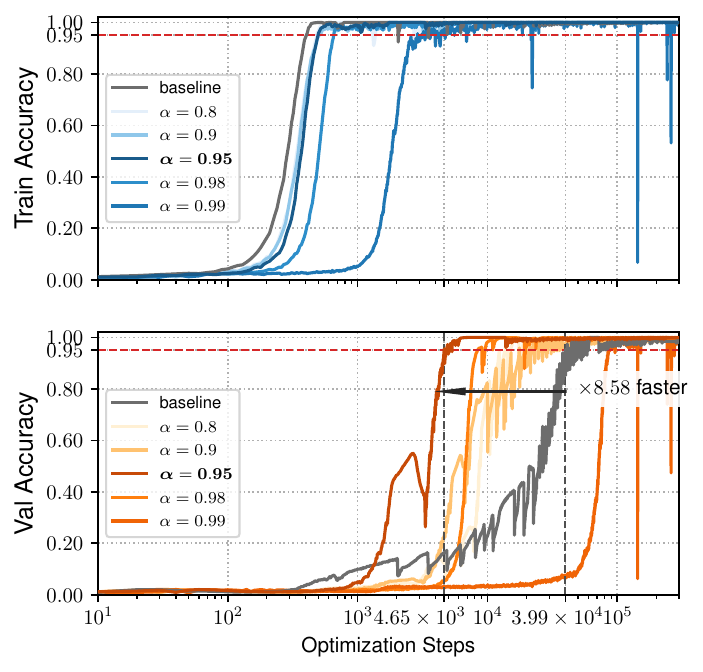}}
\hfill
\subfloat[\small Accuracy w.r.t. weight decay. \label{fig:ema:acc_wd}]{\includegraphics[width=\figwidtha]{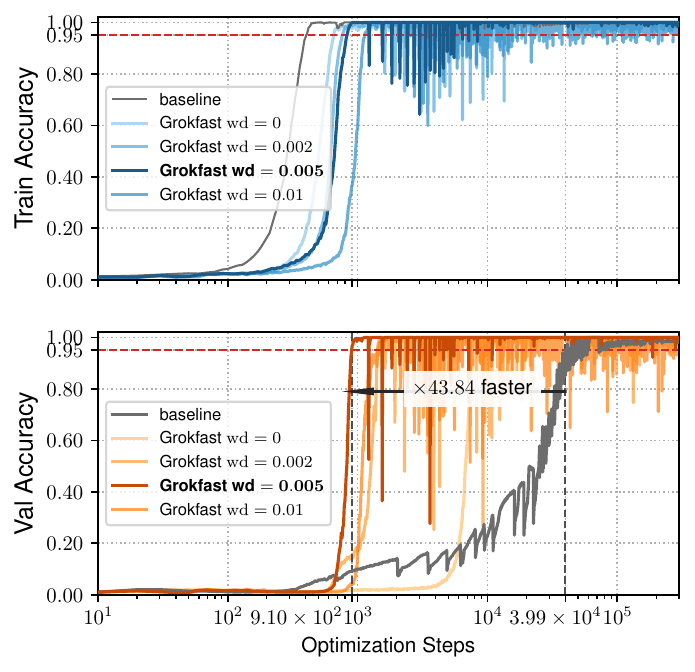}}
\caption{%
\textbf{Acceleration of delayed generation with \textsc{Grokfast-EMA} (\textsc{Grokfast}).}
The task is the same modular multiplication as in Figure~\ref{fig:ma}.
The amount of acceleration relies on three hyperparameters, the amplifier gain $\lambda\,$, the window size $w\,$, and the weight decay $\mathrm{wd}\,$.
Figures~\ref{fig:ema:acc} and~\ref{fig:ema:loss} use $\alpha = 0.98\,$, $\lambda = 2.0\,$, and $\mathrm{wd} = 0.005\,$.
Figures~\ref{fig:ema:acc_gain} and \ref{fig:ema:acc_window} show acceleration results when $\mathrm{wd} = 0\,$.
Figures~\ref{fig:ema:acc_gain}, \ref{fig:ema:acc_window}, and~\ref{fig:ema:acc_wd} use the same set of hyperparameters unless specified otherwise.
}
\label{fig:ema}
\end{center}
\vskip -0.2in
\end{figure}

%%%%%%%%%%%%%%%%%%%%%%%%%%%%%%%%%%%%%%%%%%%%%%%%%%%%%%%%%%%%%%%%%%%%%%%%%%%%%%%
Although the grokking phenomenon was first reported in the algorithmic dataset, Omnigrok~\citep{omnigrok} shows that such behavior can also be observed in a diverse set of tasks with larger and more complex datasets.
This section validates the efficacy of our accelerating algorithm, $\textsc{Grokfast}\,$, for those various tasks and models that exhibit the grokking phenomenon.

%%%%%%%%%%%%%%%%%%%%%%%%%%%%%%%%%%%%%%%%%%%%%%%%%%%%%%%%%%%%%%%%%%%%%%%%%%%%%%%
\subsection{Algorithmic Data} % (Modular Multiplication)}
We first train the same task with the same model as in Section~\ref{sec:2_pre} using our new Algorithm~\ref{alg:main}.
This is the same modular multiplication task devised to report the grokking phenomenon~\citep{grok}.
Consuming much smaller computational resources ($\times 50$ less) compared to Algorithm~\ref{alg:avg}, exponential moving average effectively captures the slow variation of the gradients necessary for accelerating the delayed generalization.
Under the grokking phenomenon, the validation loss of the model first increases before it decreases again later during the late generalization stage as depicted in Figure~\ref{fig:ema:loss} (baseline).
This is well aligned with our state transition model interpretation in \textbf{Q2} of Section~\ref{sec:2_pre:discussion}.
High difference in the validation loss implies that the \emph{generalization route} $\mathrm{B} \rightarrow \mathrm{C}$ is much longer than the \emph{overfitting route} $\mathrm{A} \rightarrow \mathrm{B}$ in the parameter space.
In contrast, models under our \textsc{Grokfast} training algorithms show significantly smaller peak in the validation loss as shown in Figures~\ref{fig:figure_one:result_loss} and~\ref{fig:ema:loss}.
This implies that \textsc{Grokfast} effectively keeps the \emph{generalization route} $\mathrm{B} \rightarrow \mathrm{C}$ close to the global optimum at state $\mathrm{C}\,$.
We will revisit these conjectures in Section~\ref{sec:5_disc} with more visualization.

We also conduct ablation studies to find out the effect of hyperparameters $\lambda\,$, $\alpha\,$, and weight decay for our \textsc{Grokfast} algorithm with an EMA filter.
The optimal hyperparameters are found with grid search as in Section~\ref{sec:2_pre}.
Figures~\ref{fig:ema:acc_gain} through~\ref{fig:ema:acc_wd} summarizes the results.
Recalling that our main idea is at the design of a low-pass filter, the momentum parameter $\alpha$ of Algorithm~\ref{alg:main} as well as the window size parameter $w$ of Algorithm~\ref{alg:avg} are equivalent to the cutoff frequency of the underlying filter.
Experiments in Figures~\ref{fig:ema:acc_gain} through~\ref{fig:ema:acc_wd} as well as those in Figures~\ref{fig:ma} and~\ref{fig:discussion_ma:wd} show that there exists a sweet spot in cutoff frequency that corresponds to the generalization-inducing gradient signal.
From our empirical studies, we recommend $\lambda \in [0.1, 5]$ and $\alpha \in [0.8, 0.99]\,$.
The weight decay is, like in typical optimization problems, dependent on the task of interest.

%%%%%%%%%%%%%%%%%%%%%%%%%%%%%%%%%%%%%%%%%%%%%%%%%%%%%%%%%%%%%%%%%%%%%%%%%%%%%%%
\begin{figure}[t]
\newcommand{\figwidth}{0.495\linewidth}
\begin{minipage}{0.62\linewidth}
\begin{center}
    \subfloat[\small Accuracy of MNIST. \label{fig:mnist:loss_gain}]{\includegraphics[width=\figwidth]{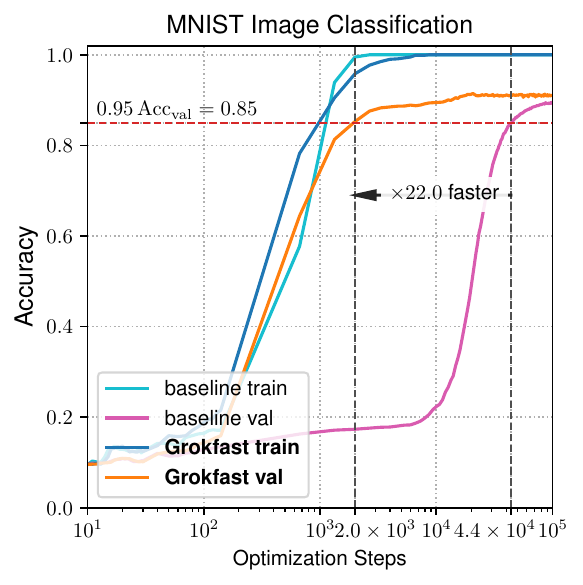}}
    \hfill
    \subfloat[\small Loss of MNIST. \label{fig:mnist:acc_gain}]{\includegraphics[width=\figwidth]{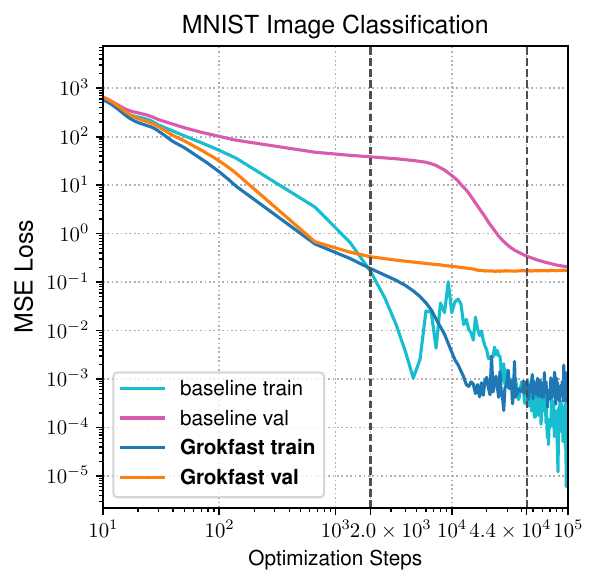}}
    \caption{%
    MNIST results with a three-layer MLP.
    Grokking phenomenon is almost gone with proper hyperparameters.
    }
    \label{fig:mnist}%
\end{center}
\end{minipage}
\hfill
\begin{minipage}{0.363\linewidth}
\vspace{0.5em}
\begin{center}
    \includegraphics[width=\linewidth]{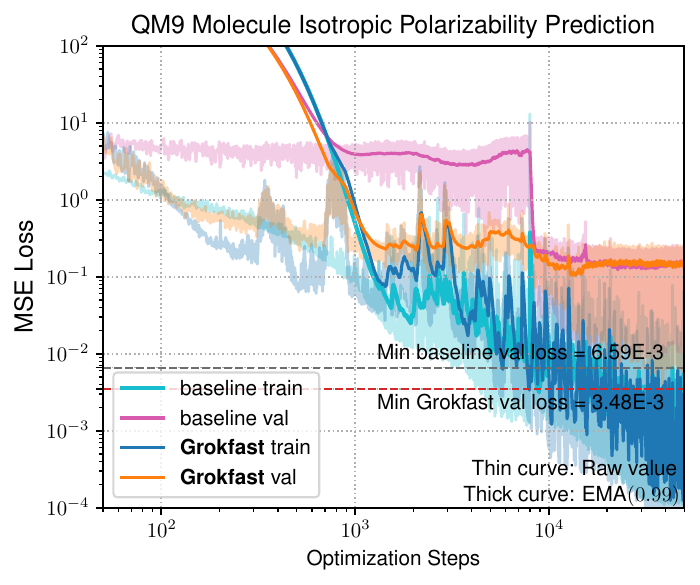}\\  
    % \vspace{2.0em}
    \caption{%
    QM9 dataset results with a GCNN.
    \textsc{Grokfast} achieves faster and better convergence.
    }
    \label{fig:qm9}%
\end{center}
\end{minipage}
\vskip -0.2in
\end{figure}
%%%%%%%%%%%%%%%%%%%%%%%%%%%%%%%%%%%%%%%%%%%%%%%%%%%%%%%%%%%%%%%%%%%%%%%%%%%%%%%
\subsection{MNIST}
Besides the simple algorithmic reasoning task, where the data is relatively simple, \cite{omnigrok} report the similar delayed generalization can also be observed in many conventional tasks if the model goes through a special treatment.
To demonstrate the generalizability of our \textsc{Grokfast} modification of the optimization process, we try to accelerate the speed of generalization under those reported models and tasks.
The first is a three-layer ReLU-MLP trained for MNIST classification task~\citep{mnist} which exhibits the grokking phenomenon.
Figure~\ref{fig:mnist} summarizes the results, showing that our Algorithm~\ref{alg:main} successfully accelerate the delayed generalization.
With $\alpha = 0.8\,$, $\lambda = 0.1\,$, and $\mathrm{wd} = 2.0\,$, the delay until grokking is reduced by $\times 22.0\,$.
Moreover, the final evaluation accuracy becomes higher from 89.8\% to 91.2\%.

\subsection{QM9}
In the next experiment, we train a graph convolutional neural network (GCNN) trained for a molecule dataset QM9~\citep{qm9,qm9_2}.
Since this task does not have an accuracy measure to compare the speed of convergence, we instead compare the convergence speed of the validation loss.
With the same setup as in Omnigrok~\citep{omnigrok}, elaborated in Appendix~\ref{sec:b_impl_detail}, we apply Algorithm~\ref{alg:main} with $\alpha = 0.9\,$, $\lambda = 1.0\,$, and $\mathrm{wd} = 0.01$ to obtain the results in Figure~\ref{fig:qm9}.
The validation loss drops faster \emph{and} by a larger margin under \textsc{Grokfast}.

%%%%%%%%%%%%%%%%%%%%%%%%%%%%%%%%%%%%%%%%%%%%%%%%%%%%%%%%%%%%%%%%%%%%%%%%%%%%%%%
\begin{figure}[t]
\newcommand{\figwidth}{0.495\linewidth}
\begin{center}
\subfloat[\small Accuracy on IMDb sentiment analysis. \label{fig:imdb:acc}]{\includegraphics[width=\figwidth]{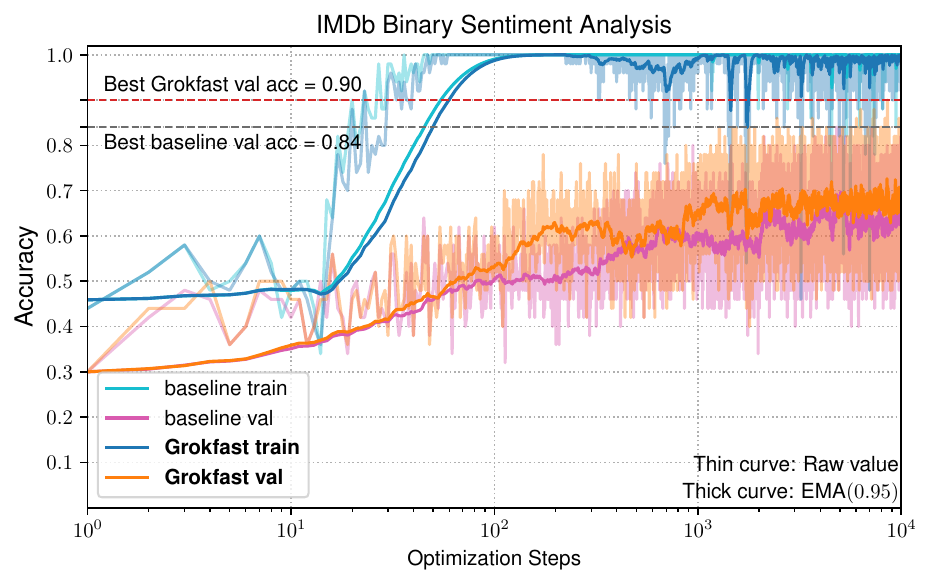}}
\hfill
\subfloat[\small Loss on IMDb sentiment analysis. \label{fig:imdb:loss}]{\includegraphics[width=\figwidth]{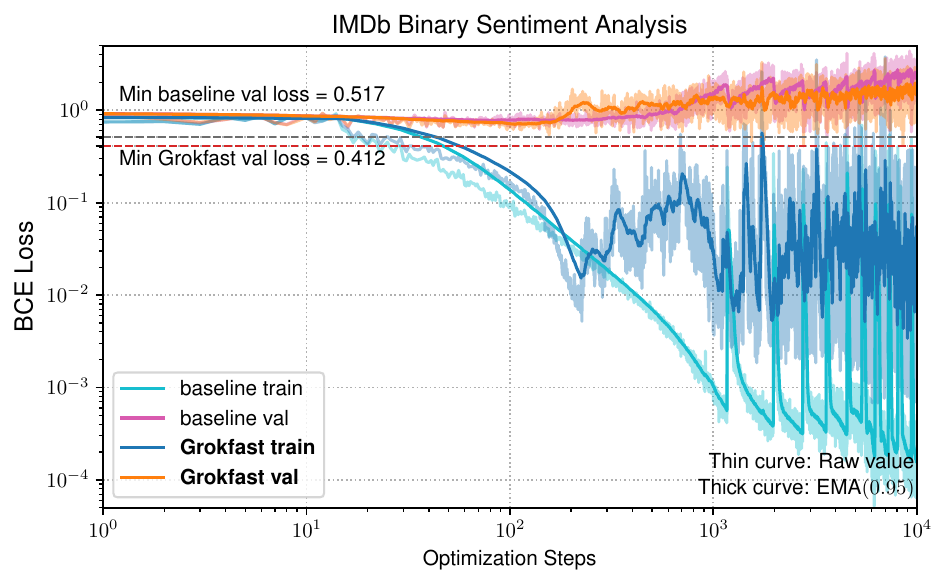}}
\caption{%
IMDb results with a two-layer LSTM.
LSTM exhibits grokking phenomenon if the training begins with larger weight values at initialization.
\textsc{Grokfast} algorithm produces faster generalization with higher validation error and lower validation loss.
We show the exponential moving average (momentum 0.95) of the curves as thick lines for clearer visualization of the trend.
}
\label{fig:imdb}
\end{center}
\vskip -0.2in
\end{figure}
%%%%%%%%%%%%%%%%%%%%%%%%%%%%%%%%%%%%%%%%%%%%%%%%%%%%%%%%%%%%%%%%%%%%%%%%%%%%%%%
\subsection{IMDb}
Finally, we train a 2-layer LSTM~\citep{lstm} network for sentiment analysis in the IMDb dataset~\citep{imdb} under the grokking phenomenon~\citep{omnigrok}.
Figure~\ref{fig:imdb} compares the baseline with the model trained with an optimizer modified by Algorithm~\ref{alg:main} with $\alpha = 0.98\,$, $\lambda = 2.0\,$, and $\mathrm{wd} = 10.0$.
We visualize the convergence speed and quantitatively compare the best validation loss/accuracy.
This experiment section suggests that \textsc{Grokfast} generally boosts performance and convergence speed in diverse tasks under the grokking phenomenon.

\section{More Discussion}
\label{sec:5_disc}
%%%%%%%%%%%%%%%%%%%%%%%%%%%%%%%%%%%%%%%%%%%%%%%%%%%%%%%%%%%%%%%%%%%%%%%%%%%%%%%
\subsection{Difference between Algorithm~\ref{alg:main} and the Momentum in Typical Optimizers}
The lines 7-8 of Algorithm~\ref{alg:main} take a similar form to the momentum variable, which is frequently used in optimizers in deep learning frameworks.
However, notable differences exist:
(1) Instead of using the scaled momentum as a parameter update, we use the smoothened gradient as a \emph{residual}, which is added to the gradient before it is fed into the optimizer.
Rather, the formula is more similar to Nesterov's momentum; however, the filtering is applied \emph{before} the optimizer, which is different from typical applications of Nesterov's momentum such as NAdam~\citep{nadam}.
(2) The line 7-8 is applied to the gradients independently to the underlying optimizer.
The optimizer can be of any type unless it is of the first-order gradient descent-based.
Low-pass filtering the gradients $g(t)$ has the same effect as filtering the post-optimizer parameter updates $u(t)$ as mathematically explained in Appendix~\ref{sec:a_math} with SGD and variants, and empirically proved in the previous sections with Adam~\citep{adam} and AdamW~\citep{weight_decay} optimizers.

%%%%%%%%%%%%%%%%%%%%%%%%%%%%%%%%%%%%%%%%%%%%%%%%%%%%%%%%%%%%%%%%%%%%%%%%%%%%%%%
\subsection{Visualizing Trajectories}
In this final section, we elaborate on our \textit{state transition interpretation} of grokking introduced in Section~\ref{sec:2_pre:discussion}.
Our signal space model of the training dynamics allows us to interpret the training of a model as a random drift of the state in the parameter space.
To visualize the dynamics, we collect all the 423k parameters of the Transformer decoder~\citep{transformer} used in the experiment in Figure~\ref{fig:figure_one} for all iterations, and conduct the PCA to obtain the most salient projection of the parameter space.
The sequence of evolving models are projected onto the space as in Figure~\ref{fig:traj:large}.

Regarding state transition interpretation of grokking, we can observe the followings:
First, Figure~\ref{fig:traj} suggests that, in the baseline setup, the model drifts through a significantly longer pathway from the overfitting (state $\mathrm{B}\,$, 500 steps) to its full generalization (state $\mathrm{C}\,$, 300k steps), compared to the initial state (state $\mathrm{A}\,$, 0 steps) to the overfitting state (state $\mathrm{B}$).
However, under \textsc{Grokfast}, the ratio between the two distances $\overline{\mathrm{AB}}$ and $\overline{\mathrm{BC}}$ in the parameter space becomes more even.
This is further acknowledged by Figure~\ref{fig:traj:dist} and Table~\ref{tab:traj} showing the distances between the models at each state.
Moreover, Table~\ref{tab:traj} suggests that the distances $\overline{\mathrm{AC}}$ between the initial and the final state becomes much ($\times 16$) shorter with our \textsc{Grokfast} algorithm. 
Although the generalization accuracy, the training accuracy, the training loss, and the validation loss at the final state (state $\mathrm{C}$) are similar in both the baseline and \textsc{Grokfast} as showcased in Figure~\ref{fig:figure_one}, we cannot simply say that the states $\mathrm{C}$ of baseline and of \textsc{Grokfast} belong to the same network state.
Likewise the state $\mathrm{B}$ of the baseline and of \textsc{Grokfast} are different.
Figure~\ref{fig:traj:dist} shows average deviation of parameter weights from the initialization point during training of the model under grokking phenomenon.
Interestingly, at achieving overfitting at state $\mathrm{B}\,$, the model under our algorithm deviates $\times 8$ further from the initial point than the baseline does, with $\times 5$ smaller standard deviation in distances from the initial state $\mathrm{A}\,$.
This suggests that although both algorithms exhibit overfitted behavior at state $\mathrm{B}\,$, intermediate model instances at these states form distinct set of parameters with possibly different topologies.
These observations support our interpretation to regard the grokking phenomenon as a state transition between at least three distinct states.
The role of \textsc{Grokfast} is then to provide supervision towards an alternative optimum much nearer from the initial points than the baseline optimizer.

Lastly, the model trained with our \textsc{Grokfast} algorithm shows hundredfold smaller variances of the distances than the baseline as claimed in Table~\ref{tab:traj}.
This implies that training under \textsc{Grokfast} algorithm is much more deterministic than under typical first-order optimizers.
This is possibly related to the similarity between the low-pass filtered gradients from small minibatches with normal gradients from larger minibatches.
However, we have also demonstrated in Section~\ref{sec:2_pre:discussion} that using only the slow, more deterministic component of gradients and completely neglecting the original gradients lead to instability.
Therefore, further investigation is needed to find out the source and the role of this determinism from our \textsc{Grokfast} algorithm, and the reason of its benefits when jointly applied with the faster, more stochastic gradients from baseline optimizers.

%%%%%%%%%%%%%%%%%%%%%%%%%%%%%%%%%%%%%%%%%%%%%%%%%%%%%%%%%%%%%%%%%%%%%%%%%%%%%%%
\begin{figure}[t]
\newcommand{\figwidth}{0.495\linewidth}
\begin{minipage}{0.68\linewidth}
\begin{center}
    \subfloat[\small Parameter trajectories. \label{fig:traj:large}]{\includegraphics[width=\figwidth]{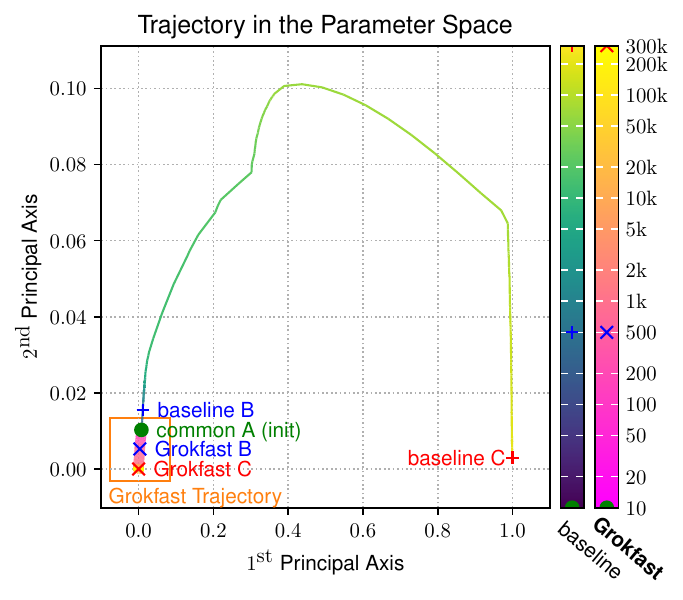}}
    \hfill
    % \subfloat[\small Magnification of the {\color[HTML]{f06406} orange} box. \label{fig:traj:mag}]{\includegraphics[width=\figwidth]{figures/vis_traj/grokfast_trajectory_magnified.pdf}}
    \subfloat[\small Deviation from initial weights. \label{fig:traj:dist}]{\includegraphics[width=\figwidth]{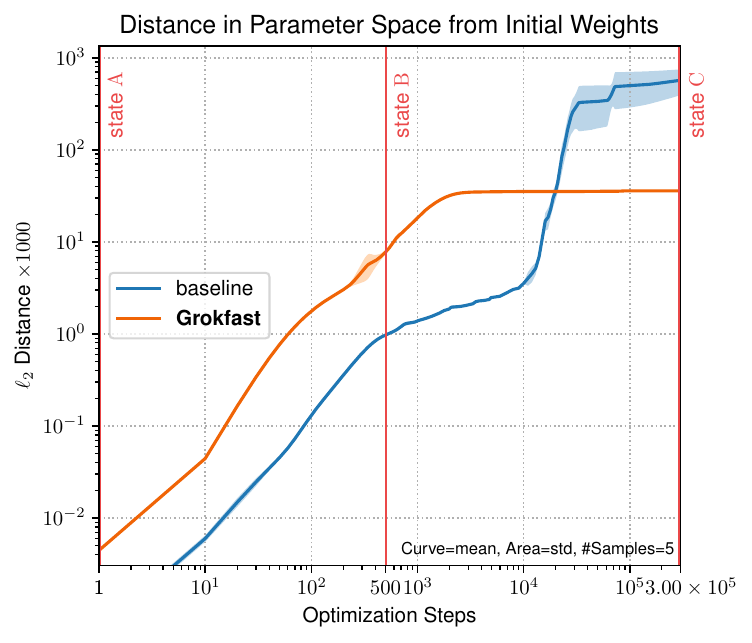}}
    \caption{%
    Trajectories of model parameters from experiments of Figure~\ref{fig:figure_one} projected onto two principal axes of the PCA of the intermediate model parameters of the baseline.
    The two models travel along distinct pathways in the parameter space with different pace.
    }
    \label{fig:traj}%
\end{center}
\end{minipage}
\hfill
\begin{minipage}{0.305\linewidth}
\begin{center}
    \captionof{table}{%
    Parameter space distances between the intermediate models of experiments in Figure~\ref{fig:figure_one}.
    Each state corresponds to each of the markers of Figure~\ref{fig:traj}.
    Average and standard deviations of five instances are shown. 
    \textsc{Grokfast} converges to a nearer point in the parameter space, and the baseline model travels longer to reach the final state.
    }
    \label{tab:traj}
    \resizebox{\linewidth}{!}{%
    \begin{tabular}{c|c|c}
    \toprule
    \multirow{2}{*}{State Pair} & \multicolumn{2}{c}{$\ell_{2}$ Distances ($\times 1000$)} \\
    \cline{2-3}
    \\[-0.9em]
    & Baseline & \textsc{Grokfast} \\
    \midrule
    $\overline{\mathrm{AB}}$ & $\phantom{0}0.97 \pm 2.2\phantom{00}$ & $\phantom{0}7.7 \pm 0.42$ \\
    $\overline{\mathrm{BC}}$ & $563.7 \pm 199.4$ & $15.3 \pm 0.84$ \\
    $\overline{\mathrm{AC}}$ & $570.7 \pm 199.8$ & $35.8 \pm 0.98$ \\
    \bottomrule
    \end{tabular}%
    }
\end{center}
\end{minipage}
\vskip -0.2in
\end{figure}
%%%%%%%%%%%%%%%%%%%%%%%%%%%%%%%%%%%%%%%%%%%%%%%%%%%%%%%%%%%%%%%%%%%%%%%%%%%%%%%

\section{Related Work}
\label{sec:6_survey}
\paragraph{Grokking.}
The recently discovered grokking phenomenon~\citep{grok} signifies one possibility of overparameterized neural networks generalizing (and reasoning) beyond memorization of the given dataset.
Most of the works, thereafter, focus on identifying its mechanism.
Recent studies associated grokking with a double descent phenomenon in the DNN mapping geometry training dynamics~\citep{grok_dd_1}, the speed of pattern learning
~\citep{grok_dd_2}, and the sizes of models and datasets~\citep{grok_dd_3}, wherein the validation error initially increases and then decreases with the expansion of model parameters~\citep{dd1, dd3}.
To investigate the internal roles of each model component during grokking, \cite{grok_mechanistic} employed mechanistic interpretability, a technique from the XAI domain, and revealed that grokking may not occur abruptly but rather exhibits an internal progress measure.
Their assertion posits that the model captures slow, generalizing patterns, underscoring the critical role of proper optimization.
Interestingly, while weight decay amplifies the double descent effect \citep{dd4_wd}, it contributes to enhanced generalization in grokking scenarios~\citep{grok}.
\cite{omnigrok} found more examples of grokking in various tasks and analyzed their training mechanism through loss landscape analysis.
\cite{grok_optimizer} found a similarity between grokking and the slingshot mechanism in adaptive optimizers.
\cite{hidden_progress} argued that optimizers reach delayed generalization by amplifying sparse solutions through hidden progress.
Regularizers such as weight decay~\citep{grok_mechanistic} and the choice of the optimizer~\citep{tug:aet} are highlighted as important factors in training a model that groks.
Our work is a continuation of this discussion by providing a generalizable tool for the practical study of the grokking phenomenon.
Through our discussion, we suggest a state transition model of the grokking and visualize the trajectory of the model weights in the parameter space during training.

%%%%%%%%%%%%%%%%%%%%%%%%%%%%%%%%%%%%%%%%%%%%%%%%%%%%%%%%%%%%%%%%%%%%%%%%%%%%%%%
\paragraph{Optimization techniques.}
At the core of the study of grokking lies optimization techniques~\citep{grok_optimizer}.
Studies have shown that generalization patterns of the model vary significantly depending on various optimization methods~\citep{grok, grok_ma}. 
\cite{grok} demonstrated that various factors related to optimization, such as (mini-)batch training~\citep{mini_batch}, the choice of optimizer, weight decay~\citep{weight_decay}, noise injection~\citep{noise}, dropout~\citep{dropout}, and learning rate, influence the model's grokking pattern. 
\cite{grok_mechanistic} argued that grokking does not occur without proper regularization.
Further, they demonstrated that techniques such as weight decay, L2 norm, and dropout induce grokking, but L1 norm does not.
On the other hand, \cite{grok_optimizer} argued that grokking \textit{can} occur without explicit regularization, attributing this to the optimizer's ``visible slingshot mechanism'' acting as an implicit regularizer.
\cite{tug:aet} suggested using a larger learning rate for the input embedding segment, facilitating unified learning of the generalization pattern.
Unlike these revisiting of the known training techniques, we started from a state space model and a dual domain of the training dynamics.
This led us to develop an optimizer augmentation algorithm, \textsc{Grokfast}, that can be applied to any existing first-order optimizers to accelerate the grokking effect for practical usage. 

% \section{Limitation}
% \label{sec:5_limit}
% \input{sections/5_limitation}

\section{Conclusion}
\label{sec:7_conclusion}
Our reinterpretation of the deviation of each model parameter into a random signal over training iteration allows us to separate gradient updates into fast-varying and slow-varying components.
By amplifying the latter with low-pass filtering, we can bring forward the moment of sudden late generalization, i.e., grokking, reducing the number of required training iterations by up to $\times 50\,$.
Our comprehensive experiments and analyses suggest that our state space interpretation and the frequency representation of the training dynamics is useful for studying the grokking phenomenon.

% \section*{References}
{
\small
\bibliography{references.bib}

\begin{thebibliography}{29}
\providecommand{\natexlab}[1]{#1}
\providecommand{\url}[1]{\texttt{#1}}
\expandafter\ifx\csname urlstyle\endcsname\relax
  \providecommand{\doi}[1]{doi: #1}\else
  \providecommand{\doi}{doi: \begingroup \urlstyle{rm}\Url}\fi

\bibitem[Barak et~al.(2022)Barak, Edelman, Goel, Kakade, Malach, and Zhang]{hidden_progress}
Boaz Barak, Benjamin Edelman, Surbhi Goel, Sham Kakade, Eran Malach, and Cyril Zhang.
\newblock Hidden progress in deep learning: Sgd learns parities near the computational limit.
\newblock In \emph{NeurIPS}, 2022.

\bibitem[Belkin et~al.(2019)Belkin, Hsu, Ma, and Mandal]{dd3}
Mikhail Belkin, Daniel Hsu, Siyuan Ma, and Soumik Mandal.
\newblock Reconciling modern machine-learning practice and the classical bias--variance trade-off.
\newblock \emph{Proceedings of the National Academy of Sciences}, 116\penalty0 (32):\penalty0 15849--15854, 2019.

\bibitem[Davies et~al.(2023)Davies, Langosco, and Krueger]{grok_dd_2}
Xander Davies, Lauro Langosco, and David Krueger.
\newblock Unifying grokking and double descent.
\newblock \emph{arXiv preprint arXiv:2303.06173}, 2023.

\bibitem[Deng(2012)]{mnist}
Li~Deng.
\newblock The mnist database of handwritten digit images for machine learning research [best of the web].
\newblock \emph{IEEE signal processing magazine}, 29\penalty0 (6):\penalty0 141--142, 2012.

\bibitem[Dozat(2016)]{nadam}
Timothy Dozat.
\newblock Incorporating nesterov momentum into adam.
\newblock In \emph{ICLR Workshop}, 2016.

\bibitem[Gromov(2023)]{grok_ma}
Andrey Gromov.
\newblock Grokking modular arithmetic.
\newblock \emph{arXiv preprint arXiv:2301.02679}, 2023.

\bibitem[{He} et~al.(2015){He}, {Zhang}, {Ren}, and {Sun}]{kaiming_init}
Kaiming {He}, Xiangyu {Zhang}, Shaoqing {Ren}, and Jian {Sun}.
\newblock {Delving Deep into Rectifiers: Surpassing Human-Level Performance on ImageNet Classification}.
\newblock In \emph{ICCV}, 2015.

\bibitem[{Hendrycks} and {Gimpel}(2016)]{gelu}
Dan {Hendrycks} and Kevin {Gimpel}.
\newblock {{Gaussian} Error Linear Units ({GELUs})}.
\newblock \emph{arXiv preprint arXiv:1606.08415}, 2016.

\bibitem[Hochreiter and Schmidhuber(1997)]{lstm}
Sepp Hochreiter and J\"{u}rgen Schmidhuber.
\newblock Long short-term memory.
\newblock \emph{Neural Comput.}, 9\penalty0 (8):\penalty0 1735–1780, nov 1997.
\newblock ISSN 0899-7667.
\newblock \doi{10.1162/neco.1997.9.8.1735}.
\newblock URL \url{https://doi.org/10.1162/neco.1997.9.8.1735}.

\bibitem[Huang et~al.(2024)Huang, Hu, Han, Liu, and Sun]{grok_dd_3}
Yufei Huang, Shengding Hu, Xu~Han, Zhiyuan Liu, and Maosong Sun.
\newblock Unified view of grokking, double descent and emergent abilities: A perspective from circuits competition.
\newblock \emph{arXiv preprint arXiv:2402.15175}, 2024.

\bibitem[Humayun et~al.(2023)Humayun, Balestriero, and Baraniuk]{grok_dd_1}
Ahmed~Imtiaz Humayun, Randall Balestriero, and Richard Baraniuk.
\newblock Training dynamics of deep network linear regions.
\newblock \emph{arXiv preprint arXiv:2310.12977}, 2023.

\bibitem[Kingma and Ba(2014)]{adam}
Diederik~P Kingma and Jimmy Ba.
\newblock Adam: A method for stochastic optimization.
\newblock \emph{arXiv preprint arXiv:1412.6980}, 2014.

\bibitem[{Lei Ba} et~al.(2016){Lei Ba}, {Kiros}, and {Hinton}]{layernorm}
Jimmy {Lei Ba}, Jamie~Ryan {Kiros}, and Geoffrey~E. {Hinton}.
\newblock {Layer Normalization}.
\newblock In \emph{Advances in NIPS 2016 Deep Learning Symposium}, 2016.

\bibitem[Li et~al.(2014)Li, Zhang, Chen, and Smola]{mini_batch}
Mu~Li, Tong Zhang, Yuqiang Chen, and Alexander~J Smola.
\newblock Efficient mini-batch training for stochastic optimization.
\newblock In \emph{Proceedings of the 20th ACM SIGKDD international conference on Knowledge discovery and data mining}, pages 661--670, 2014.

\bibitem[Liu et~al.(2022{\natexlab{a}})Liu, Kitouni, Nolte, Michaud, Tegmark, and Williams]{tug:aet}
Ziming Liu, Ouail Kitouni, Niklas Nolte, Eric~J Michaud, Max Tegmark, and Mike Williams.
\newblock Towards understanding grokking: An effective theory of representation learning.
\newblock In \emph{NeurIPS}, 2022{\natexlab{a}}.

\bibitem[Liu et~al.(2022{\natexlab{b}})Liu, Michaud, and Tegmark]{omnigrok}
Ziming Liu, Eric~J Michaud, and Max Tegmark.
\newblock Omnigrok: Grokking beyond algorithmic data.
\newblock \emph{arXiv preprint arXiv:2210.01117}, 2022{\natexlab{b}}.

\bibitem[Loshchilov and Hutter(2018)]{weight_decay}
Ilya Loshchilov and Frank Hutter.
\newblock Decoupled weight decay regularization.
\newblock In \emph{ICLR}, 2018.

\bibitem[Maas et~al.(2011)Maas, Daly, Pham, Huang, Ng, and Potts]{imdb}
Andrew Maas, Raymond~E Daly, Peter~T Pham, Dan Huang, Andrew~Y Ng, and Christopher Potts.
\newblock Learning word vectors for sentiment analysis.
\newblock In \emph{Proceedings of the 49th annual meeting of the association for computational linguistics: Human language technologies}, pages 142--150, 2011.

\bibitem[Nakkiran et~al.(2021)Nakkiran, Kaplun, Bansal, Yang, Barak, and Sutskever]{dd1}
Preetum Nakkiran, Gal Kaplun, Yamini Bansal, Tristan Yang, Boaz Barak, and Ilya Sutskever.
\newblock Deep double descent: Where bigger models and more data hurt.
\newblock \emph{Journal of Statistical Mechanics: Theory and Experiment}, 2021\penalty0 (12):\penalty0 124003, 2021.

\bibitem[Nanda et~al.(2023)Nanda, Chan, Lieberum, Smith, and Steinhardt]{grok_mechanistic}
Neel Nanda, Lawrence Chan, Tom Lieberum, Jess Smith, and Jacob Steinhardt.
\newblock Progress measures for grokking via mechanistic interpretability.
\newblock \emph{arXiv preprint arXiv:2301.05217}, 2023.

\bibitem[Paszke et~al.(2019)Paszke, Gross, Massa, Lerer, Bradbury, Chanan, Killeen, Lin, Gimelshein, Antiga, et~al.]{pytorch}
Adam Paszke, Sam Gross, Francisco Massa, Adam Lerer, James Bradbury, Gregory Chanan, Trevor Killeen, Zeming Lin, Natalia Gimelshein, Luca Antiga, et~al.
\newblock Pytorch: An imperative style, high-performance deep learning library.
\newblock \emph{Advances in neural information processing systems}, 32, 2019.

\bibitem[Pezeshki et~al.(2022)Pezeshki, Mitra, Bengio, and Lajoie]{dd4_wd}
Mohammad Pezeshki, Amartya Mitra, Yoshua Bengio, and Guillaume Lajoie.
\newblock Multi-scale feature learning dynamics: Insights for double descent.
\newblock In \emph{ICML}, 2022.

\bibitem[Power et~al.(2022)Power, Burda, Edwards, Babuschkin, and Misra]{grok}
Alethea Power, Yuri Burda, Harri Edwards, Igor Babuschkin, and Vedant Misra.
\newblock Grokking: Generalization beyond overfitting on small algorithmic datasets.
\newblock \emph{arXiv preprint arXiv:2201.02177}, 2022.

\bibitem[Ramakrishnan et~al.(2014)Ramakrishnan, Dral, Rupp, and von Lilienfeld]{qm9_2}
Raghunathan Ramakrishnan, Pavlo~O. Dral, Matthias Rupp, and O.~Anatole von Lilienfeld.
\newblock Quantum chemistry structures and properties of 134 kilo molecules.
\newblock \emph{Scientific Data}, 1\penalty0 (1):\penalty0 140022, Aug 2014.
\newblock ISSN 2052-4463.
\newblock \doi{10.1038/sdata.2014.22}.
\newblock URL \url{https://doi.org/10.1038/sdata.2014.22}.

\bibitem[Ruddigkeit et~al.(2012)Ruddigkeit, van Deursen, Blum, and Reymond]{qm9}
Lars Ruddigkeit, Ruud van Deursen, Lorenz~C. Blum, and Jean-Louis Reymond.
\newblock Enumeration of 166 billion organic small molecules in the chemical universe database {GDB-17}.
\newblock \emph{Journal of Chemical Information and Modeling}, 52\penalty0 (11):\penalty0 2864--2875, Nov 2012.
\newblock ISSN 1549-9596.
\newblock \doi{10.1021/ci300415d}.
\newblock URL \url{https://doi.org/10.1021/ci300415d}.

\bibitem[Srivastava et~al.(2014)Srivastava, Hinton, Krizhevsky, Sutskever, and Salakhutdinov]{dropout}
Nitish Srivastava, Geoffrey Hinton, Alex Krizhevsky, Ilya Sutskever, and Ruslan Salakhutdinov.
\newblock Dropout: a simple way to prevent neural networks from overfitting.
\newblock \emph{The journal of machine learning research}, 15\penalty0 (1):\penalty0 1929--1958, 2014.

\bibitem[Thilak et~al.(2022)Thilak, Littwin, Zhai, Saremi, Paiss, and Susskind]{grok_optimizer}
Vimal Thilak, Etai Littwin, Shuangfei Zhai, Omid Saremi, Roni Paiss, and Joshua~M Susskind.
\newblock The slingshot mechanism: An empirical study of adaptive optimizers and the grokking phenomenon.
\newblock In \emph{Has it Trained Yet? NeurIPS 2022 Workshop}, 2022.

\bibitem[{Vaswani} et~al.(2017){Vaswani}, {Shazeer}, {Parmar}, {Uszkoreit}, {Jones}, {Gomez}, {Kaiser}, and {Polosukhin}]{transformer}
Ashish {Vaswani}, Noam {Shazeer}, Niki {Parmar}, Jakob {Uszkoreit}, Llion {Jones}, Aidan~N. {Gomez}, Lukasz {Kaiser}, and Illia {Polosukhin}.
\newblock {Attention Is All You Need}.
\newblock In \emph{NIPS}, 2017.

\bibitem[Zur et~al.(2009)Zur, Jiang, Pesce, and Drukker]{noise}
Richard~M Zur, Yulei Jiang, Lorenzo~L Pesce, and Karen Drukker.
\newblock Noise injection for training artificial neural networks: A comparison with weight decay and early stopping.
\newblock \emph{Medical physics}, 36\penalty0 (10):\penalty0 4810--4818, 2009.

\end{thebibliography}
}

%%%%%%%%%%%%%%%%%%%%%%%%%%%%%%%%%%%%%%%%%%%%%%%%%%%%%%%%%%%%

\appendix

\section{Frequency Responses of the Parameter Updates under Grokfast}
\label{sec:a_math}
In Section~\ref{sec:1_intro}, we have silently assumed that under a first-order optimizer $u(g(t), t)\,$, amplifying the low-frequency components of the gradient signal $g(t)$ of an arbitrary parameter $\theta(t)$ over a discrete timestep $t$ has the same effect of amplifying the low-frequency component of the parameter updates $u(t) = u(g(t), t)\,$.
This section mathematically elaborates on the effect of gradient filters $h(t)$ to the parameter update signals $u(t)$ in the most frequently-used type of optimizers: SGD with momentum.

%%%%%%%

Stochastic gradient descent with optional momentum term is the simplest and the most widely used optimization algorithm in the deep learning communities.
Here, the parameter update $u(t) = \theta(t + 1) - \theta(t)$ of a parameter $\theta(t)$ at timestep $t$ and its intermediate momentum $m(t)$ is defined as:
\begin{align}
    \label{eq:a_math:sgdm_momentum}
    m(t) &= \mu m(t - 1) + (1 - \tau) g(t)\,, \\
    \label{eq:a_math:sgdm_full}
    u(t) &= - \eta m(t)\,,
\end{align}
where $\mu$ is the scalar momentum, $\tau$ is the dampening constant for the momentum, and $\eta$ is the learning rate.
This class of optimizers can be thought of as linear systems with state $m(t)$ that receives an input $g(t)$ to produce an output $u(t)\,$.

To compare the difference between the frequency responses of the parameter update $u(t)$ and of the modified update $\hat{u}(t)$ in equation~\eqref{eq:2_pre:update_lpf_time}, we can think of an equivalent filter $\hat{h}(t)$ defined to satisfy the following relationship in addition to equations~\eqref{eq:2_pre:lpf_time} and~\eqref{eq:2_pre:update_lpf_time}:
\begin{equation}
    \label{eq:a_math:lpf_grad_time_equiv}
    \hat{u}(t) = u(\hat{g}(t), t) = u(g(t) + h(t) * g(t), t) = u(t) + \hat{h}(t) * u(t)\,.
\end{equation}
From our assumption of the linear time-invariant, scalar filters $h(t)$ and the linear optimizer, we can deduce the equivalence between $h(t)$ and $\hat{h}(t)\,$.
The following theorem is a generalized claim that applies to any SGD-based first-order optimizers including Nesterov's momentum.
\begin{theorem}
\label{thm:linear_optim}
Let $g(t)$ be a scalar signal defined over a discrete time $t \in \{0, 1, \ldots, T\}\,$.
Let $h(t)$ be a univariate time-invariant filter defined over the same domain $t\,$.
A linear optimizer $O$ is defined as:
\begin{align}
    \label{eq:a_math:system_state}
    x(t) &= A x(t - 1) + B g(t)\,, \quad t > 0\,, \\
    \label{eq:a_math:system_output}
    u(t) &= C x(t) + D g(t)\,, \qquad\,\,\,\, t \geq 0\,,
\end{align}
with scalar coefficients $A, B, C,$ and $D\,$, and $x(0) = g(0)\,$.
The output of the system $u(t)$ is, therefore, a function of $g(t)$ and $t\,$, i.e., $u(t) = u(g(t), t)\,$.
Let the modified input $\hat{g}(t)\,$, the modified output $\hat{u}(t)\,$, and the equivalent filter $\hat{h}(t)$ be defined to satisfy the equations~\eqref{eq:2_pre:lpf_time} and~\eqref{eq:a_math:lpf_grad_time_equiv}.
Then,
\begin{equation}
    \hat{h}(t) = h(t)\,,
\end{equation}
for $t \in \{0, 1, \ldots, T\}\,$.
\end{theorem}
\begin{proof}[Proof of Theorem~\ref{thm:linear_optim}]
For simplicity, we first adopt discrete-time Fourier transform over $t \in \mathbb{Z}\,$.
That is, we assume that the signals are defined across every positive and negative integer $t\,$.
Since the value of $u(t)$ can be defined arbitrarily outside the interval $[0, T]$ without modifying the optimization algorithm, we can manually assign $g(t)$ and $x(t)$ for $t \notin [0, T]$ as:
\begin{align}
    \label{eq:a_math:g_external}
    g(t) &= 0 \qquad\qquad\qquad\quad\,\,\, t \notin [0, T]\,,\\
    \label{eq:a_math:x_external}
    x(t) &= 
        \begin{cases}
            A^{t} (1 - B) g(0) & t < 0\,, \\
            A^{t - T} x(T) & t > T\,. \\
        \end{cases}
\end{align}
Then, equations~\eqref{eq:a_math:system_state} and~\eqref{eq:a_math:system_output} hold for $t \notin [0, T]\,$.
Note that to make the optimizer $O$ stable, the scalar coefficient $A$ should satisfy $0 < A < 1\,$.
Therefore, the signals $g(t)$ and $x(t)$ are well-defined.

Consider a discrete-time Fourier transform $\mathcal{F}$ defined as:
\begin{equation}
    \label{eq:a_math:fourier}
    \mathcal{F}\{f(t)\}(\omega) =  \sum_{t = -\infty}^{\infty} f(t) e^{-i \omega t}\,.
\end{equation}
In the frequency domain, with $G(\omega) = \mathcal{F}\{g(t)\}\,$, $U(\omega) = \mathcal{F}\{u(t)\}\,$, and $X(\omega) = \mathcal{F}\{x(t)\}\,$, the optimizer $O$ can be equivalently represented as:
\begin{align}
    \label{eq:a_math:system_state_freq}
    X(\omega) &= A e^{-i \omega} X(\omega) + B G(\omega)\,, \\
    \label{eq:a_math:system_output_freq}
    U(\omega) &= C X(\omega) + D G(\omega)\,.
\end{align}
We can obtain the transfer functions $H_{\text{in-state}}$ and $H_{\text{in-out}}$ that converts $G$ to $X$ and then to $U\,$:
\begin{align}
    \label{eq:a_math:momentum_freq}
    H_{\text{in-state}}(\omega) &\coloneqq \frac{X(\omega)}{G(\omega)} = \frac{B}{1 - A e^{-i \omega}}\,, \\
    \label{eq:a_math:sgdm_full_freq}
    H_{\text{in-out}}(\omega) &\coloneqq \frac{U(\omega)}{G(\omega)} = C \frac{X(\omega)}{G(\omega)} + D = \frac{BC}{1 - A e^{-i \omega}} + D\,.
\end{align}
If the input $g(t)$ is filtered with a convolutional filter $h(t)$ and then added to itself as equations~\eqref{eq:2_pre:lpf_time} and~\eqref{eq:2_pre:lpf_freq}, the state $x(t)$ and the output $u(t)$ of the optimizer $O$ is changed accordingly while keeping equations~\eqref{eq:a_math:system_state} and~\eqref{eq:a_math:system_output} hold.
We denote $\hat{x}(t)$ and $\hat{u}(t)$ as the modified state and output of the system and $\hat{X}(\omega)$ and $\hat{U}(\omega)$ as their spectra.
If the filter $h(t)$ is causal, that is $h(t) = 0$ for $t < 0\,$, then we can similarly let $\hat{x}(0) = \hat{g}(0)$ and replace ${x}(t)$ and ${u}(t)$ with $\hat{x}(t)$ and $\hat{u}(t)$ in equations~\eqref{eq:a_math:g_external} and~\eqref{eq:a_math:x_external} to define an IIR system suitable for the infinite-window discrete-time Fourier transform $\mathcal{F}\,$:
\begin{align}
    \label{eq:a_math:system_state_freq_mod}
    \hat{X}(\omega) &= A e^{-i \omega} \hat{X}(\omega) + B \hat{G}(\omega)\,, \\
    \label{eq:a_math:system_output_freq_mod}
    \hat{U}(\omega) &= C \hat{X}(\omega) + D \hat{G}(\omega)\,.
\end{align}
Since the coefficients of the linear systems are the same, the transfer functions are identical:
\begin{align}
    \label{eq:a_math:momentum_freq_mod}
    \hat{H}_{\text{in-state}}(\omega) &\coloneqq \frac{\hat{X}(\omega)}{\hat{G}(\omega)} = \frac{B}{1 - A e^{-i \omega}} \qquad\equiv H_{\text{in-state}}(\omega)\,, \\
    \label{eq:a_math:sgdm_full_freq_mod}
    \hat{H}_{\text{in-out}}(\omega) &\coloneqq \frac{\hat{U}(\omega)}{\hat{G}(\omega)} = \frac{BC}{1 - A e^{-i \omega}} + D \equiv H_{\text{in-out}}(\omega)\,.
\end{align}

From equation~\eqref{eq:2_pre:lpf_freq}, the transfer function of the filter $H_{\text{amp}}(\omega)$ is:
\begin{equation}
    \label{eq:a_math:lpf_tf_freq}
    H_{\text{amp}}(\omega) = \frac{\hat{G}(\omega)}{G(\omega)} = 1 + H(\omega)\,,
\end{equation}
where $H(\omega) = \mathcal{F}\{h(t)\}\,$.
The equivalent post-filter $\hat{h}(t)$ defined by equation~\eqref{eq:a_math:lpf_grad_time_equiv} gives another transfer function between the outputs $u(t)$ and $\hat{u}(t)$ of the system:
\begin{equation}
    \label{eq:a_math:lpf_tf_output_freq}
    \hat{H}_{\text{amp}}(\omega) = \frac{\hat{U}(\omega)}{U(\omega)} = 1 + \hat{H}(\omega)\,.
\end{equation}
From equations~\eqref{eq:a_math:sgdm_full_freq} and~\eqref{eq:a_math:sgdm_full_freq_mod}, we have:
\begin{equation}
    \label{eq:a_math:identity_tf_inout}
    \hat{H}_{\text{amp}}(\omega) = \frac{\hat{U}(\omega)}{U(\omega)} = \frac{\hat{G}(\omega)}{G(\omega)} = H_{\text{amp}}(\omega)\,.
\end{equation}
Therefore, we get:
\begin{equation}
    \hat{H}(\omega) \equiv {H}(\omega)\,.
\end{equation}
This completes the proof.
\end{proof}
In other words, applying any filter $h(t)$ to the sequence of gradients $g(t)$ is equivalent to the same filter $h(t)$ applied to the parameter update $u(t)$ for any linear optimizer $O\,$.
This implies that a low-pass \emph{gradient} filter $h(t)$ guarantees the same low-pass property in the modified parameter update signal $\hat{u}(t)\,$.
In many off-the-shelf autograd packages such as PyTorch~\citep{pytorch}, filtering the gradients is easier and more straightforward than filtering the intermediate parameter updates.
The former only adds a few more lines to the outermost application code\footnote{See our implementation at \url{https://github.com/ironjr/grokfast}.}, whereas the latter requires full implementation of the dedicated optimizer object.
% The former only adds a few more lines to the outermost application code\footnote{Will be released.}, whereas the latter requires full implementation of the dedicated optimizer object.
Note that the above proof holds regardless of the design of the filter $h(t)$ unless there exists a one-to-one correspondence between $H_{\text{amp}}$ and $H\,$.

The followings are direct consequences of the Theorem~\ref{thm:linear_optim}.
\begin{proposition}[SGD with momentum]
\label{prop:sgdm}
Let $t \in \{0, 1, \ldots, T\}$ be a discrete timestep.
Let $g(t)$ be a sequence of gradients of a parameter $\theta$ sampled from a stochastic machine learning framework $g(t) \sim M(\theta(t), t)$ and a stochastic gradient descent optimizer $O(\mu, \tau, \eta)$ with a parameter update function $u(g(t), t) = u(t) = \theta(t + 1) - \theta(t)\,$, a momentum $\mu\,$, a damping constant $\tau\,$, and a learning rate $\eta\,$.
The parameter update $u(t)$ is, therefore, defined as:
\begin{align}
    \label{eq:a_math:sgdm_momentum_v2}
    m(t) &= \mu m(t - 1) + (1 - \tau) g(t)\,, \\
    \label{eq:a_math:sgdm_full_v2}
    u(t) &= - \eta m(t)\,,
\end{align}
with a scalar momentum term $m(t)$ for each parameter $\theta$ with $m(0) = g(0)\,$.
Let $h(t)$ be a scalar, time-invariant, convolutional gradient filter.
Let the modified input $\hat{g}(t)\,$, the modified output $\hat{u}(t)\,$, and the equivalent filter $\hat{h}(t)$ be defined to satisfy the equations~\eqref{eq:2_pre:lpf_time} and~\eqref{eq:a_math:lpf_grad_time_equiv}.
Then,
\begin{equation}
    \label{eq:prop_sgdm_claim}
    \hat{h}(t) = h(t)\,,
\end{equation}
for $t \in \{0, 1, \ldots, T\}\,$.
\end{proposition}
\begin{proof}[Proof of Proposition~\ref{prop:sgdm}]
Let $A = \mu\,$, $B = 1 - \tau\,$, $C = -\eta\,$, and $D = 0\,$.
By Theorem~\ref{thm:linear_optim}, equation~\eqref{eq:prop_sgdm_claim} holds.
\end{proof}

\begin{proposition}[SGD with Nesterov's momentum]
\label{prop:nesterov}
Let $t \in \{0, 1, \ldots, T\}$ be a discrete timestep.
Let $g(t)$ be a sequence of gradients of a parameter $\theta$ sampled from a stochastic machine learning framework $g(t) \sim M(\theta(t), t)$ and a stochastic gradient descent optimizer $O(\mu, \tau, \eta)$ with a parameter update function $u(g(t), t) = u(t) = \theta(t + 1) - \theta(t)\,$, a momentum $\mu\,$, a damping constant $\tau\,$, and a learning rate $\eta\,$.
The parameter update $u(t)$ is, therefore, defined as:
\begin{align}
    \label{eq:a_math:sgdm_momentum_nesterov}
    m(t) &= \mu m(t - 1) + (1 - \tau) g(t)\,, \\
    \label{eq:a_math:sgdm_full_nesterov}
    u(t) &= - \eta (g(t) + \mu m(t))\,,
\end{align}
with a scalar momentum term $m(t)$ for each parameter $\theta$ with $m(0) = g(0)\,$.
Let $h(t)$ be a scalar, time-invariant, convolutional gradient filter.
Let the modified input $\hat{g}(t)\,$, the modified output $\hat{u}(t)\,$, and the equivalent filter $\hat{h}(t)$ be defined to satisfy the equations~\eqref{eq:2_pre:lpf_time} and~\eqref{eq:a_math:lpf_grad_time_equiv}.
Then,
\begin{equation}
    \label{eq:prop_sgdm_claim_nesterov}
    \hat{h}(t) = h(t)\,,
\end{equation}
for $t \in \{0, 1, \ldots, T\}\,$.
\end{proposition}
\begin{proof}[Proof of Proposition~\ref{prop:nesterov}]
Let $A = \mu\,$, $B = 1 - \tau\,$, $C = -\eta\mu\,$, and $D = -\eta\,$.
By Theorem~\ref{thm:linear_optim}, equation~\eqref{eq:prop_sgdm_claim_nesterov} holds.
\end{proof}

\section{Task Details}
\label{sec:b_impl_detail}
For completeness, this section summarizes the implementation details of each task dealt in Section~\ref{sec:4_exp}.
The readers can also consult our official implementation in PyTorch~\citep{pytorch}.

%%%%%%%
\subsection{Binary Operation (Algorithmic Data)}
This is the description of algorithmic data used throughout the manuscript.
Following the first report on the grokking phenomenon~\citep{grok}, we demonstrate our acceleration algorithms with a binary operation $x \cdot y \; (\mathrm{mod}\; p)\,$, with $p=97\,$.
The network is a two-layer decoder-only Transformer~\citep{transformer} with hidden dimension of 128 and  4 heads in its attention.
The positional embedding has length of 5, and GELU~\citep{gelu} and layer normalization~\citep{layernorm} is used throughout the network.
After the Transformer blocks, the output is fed into a layer normalization and a linear output layer to return logits.
We use cross entropy loss to train the network and an Adam~\citep{adam} with betas $(\beta_{1}, \beta_{2}) = (0.9, 0.98)\,$, a constant learning rate of $10^{-3}\,$, batch size of $512\,$, and linear learning rate warmup schedule over the first 10 iterations.

%%%%%%%
\subsection{MNIST}
We train a three-layer MLP with hidden width of 200 and ReLU activations for the MNIST classification task \citep{mnist}.
Under $\times 8$ larger weight initialization than Kaiming initialization~\citep{kaiming_init}, the network is known to exhibit the grokking phenomenon~\citep{omnigrok}.
The network receives flattened grayscale images of size $28 \times 28$ and outputs 10-dimensional logits to calculate MSE losses between one-hot encoded labels.
We use the batch size of $200$ and trained with an AdamW optimizer~\citep{weight_decay} with a constant learning rate of $10^{-3}$ until $10^{5}$ training iterations.
We use a smaller subset of 1000 images from training images to train the network in order to simulate overfitting environment.
All the other hyperparameters are set by default.

%%%%%%%
\subsection{QM9}
To demonstrate the effectiveness of our algorithm on a graph convolutional neural network, we use QM9 small molecules dataset~\citep{qm9,qm9_2} to estimate the isotropic polarizability.
Our Graph ConvNet has two graph convolution layers with input channel of 11 (QM9 edge features), output channel of 16, and hidden channel of 32.
Each graph convolution is followed by a ReLU.
Each convolution layer consists of two linear layers with an internal ReLU activation with hidden channel of 32.
The output of the Graph ConvNet is a global average pooling, followed by a two-layer MLP with a ReLU and hidden channel of 32.
To simulate the overfitting environment, we use the first 100 samples from the data.
The data is again randomly split into train and validation sets with 50:50 size ratio.
We use batch size of 32, an AdamW optimizer~\citep{weight_decay} with a constant learning rate of $10^{-3}\,$.
The network is initialized with weights $\times 3$ larger than that of Kaiming initialization~\citep{kaiming_init} and trained for 50k iterations.

%%%%%%%
\subsection{IMDb}
For IMDb dataset~\citep{imdb}, we use LSTM~\citep{lstm} with two layers, embedding dimension of 64, hidden dimension of 256, and vocabulary size of 1001, including the padding token.
The network is followed by a single fully connected layer with output dimension of 1 with sigmoid activation to classify the positive/negative sentiment of each review string.
The dataset was preprocessed by tokenizing the 1000 most frequent words from the review.
The list of integer tokens are padded by zeros to form an array of reviews with the same length of 500.
The network was trained by a binary cross entropy loss and an AdamW optimizer~\citep{weight_decay} with learning rate of $3 \times 10^{-4}$ and batch size of 50.
We trained the model with the first 1000 rows from the dataset, split into train and validation sets with 75:25 size ratio.
We stopped the training at 10k iterations as shown in Figure~\ref{fig:imdb}.

\section{Autograd Implementation}
\label{sec:c_code}
We have argued that our implementation of Algorithm~\ref{alg:avg} and~\ref{alg:main} costs only a few additional lines of code.
We demonstrate this by presenting the exact code we developed with the PyTorch~\citep{pytorch} autograd package.
The readers who are interested can also consult our official implementation\footnote{\url{https://github.com/ironjr/grokfast}}.

Algorithms~\ref{alg:avg} and~\ref{alg:main} are implemented as follows:
\begin{lstlisting}[language=python,firstnumber=1,mathescape=true]
# Grokfast-MA (Algorithm 1)
def gradfilter_ma(
    m: nn.Module,
    grads: Optional[Dict[str, deque]] = None,
    window_size: int = 100,
    lamb: float = 5.0,
    filter_type: Literal['mean', 'sum'] = 'mean',
    warmup: bool = True,
) -> Dict[str, deque]:
    if grads is None:
        grads = {n: deque(maxlen=window_size) for n, p in m.named_parameters() if p.requires_grad}

    for n, p in m.named_parameters():
        if p.requires_grad:
            grads[n].append(p.grad.data.detach())

            if not warmup or len(grads[n]) == window_size:
                if filter_type == "mean":
                    avg = sum(grads[n]) / len(grads[n])
                elif filter_type == "sum":
                    avg = sum(grads[n])
                else:
                    raise ValueError(f"Unrecognized filter_type {filter_type}")
                p.grad.data = p.grad.data + avg * lamb

    return grads

# Grokfast (Algorithm 2)
def gradfilter_ema(
    m: nn.Module,
    grads: Optional[Dict[str, torch.Tensor]] = None,
    alpha: float = 0.98,
    lamb: float = 2.0,
) -> Dict[str, torch.Tensor]:
    if grads is None:
        grads = {n: p.grad.data.detach() for n, p in m.named_parameters() if p.requires_grad}

    for n, p in m.named_parameters():
        if p.requires_grad:
            grads[n] = grads[n] * alpha + p.grad.data.detach() * (1 - alpha)
            p.grad.data = p.grad.data + grads[n] * lamb

    return grads
\end{lstlisting}
This helper method can be applied to any optimization framework involving the autograd package by inserting a single line between the calculation of the gradients and the optimizer call as follows:
\begin{lstlisting}[language=python,firstnumber=1,mathescape=true]
# ... any initialization code before starting the training loop.
grads = None

# Training loop.
for batch in dataloader:
    model.zero_grad()
    output = model(batch)
    loss = criteria(output)

    # Calculate the gradients.
    loss.backward()

    # Option 1: Grokfast (has argument alpha, lamb)
    grads = gradfilter_ema(model, grads=grads, alpha=alpha, lamb=lamb)
    # Option 2: Grokfast-MA (has argument window_size, lamb)
    # grads = gradfilter_ma(model, grads=grads, window_size=window_size, lamb=lamb)

    # Call the optimizer.
    optimizer.step()

    # ... any additional logging codes.
\end{lstlisting}
Note that line 2 and line 14 in the code above are the only modification we made.

\section{Time and Memory Requirements}
\label{sec:d_results}
%%%%%%%
\begin{table}[t]
\newcommand{\figwidth}{0.49\linewidth}
\begin{minipage}{\linewidth}
    \caption{%
    Quantitative results of \textsc{Grokfast} with a Transformer decoder trained for the algorithmic data (modular multiplication).
    The experiments corresponds to that of Figure~\ref{fig:figure_one} and~\ref{fig:ema:acc} \& \ref{fig:ema:loss}.
    }
    \label{tab:appx_exp_time:alg}
    \resizebox{\linewidth}{!}{%
    \begin{tabular}{lcccc}
    \toprule
    Algorithm & Iterations @ 95\% Val. Acc. & Wall Clock Time @ 95\% Val. Acc. (s) & VRAM (MB) & Latency Per Iteration (s) \\
    \midrule
    Baseline & $39890$ & $5984$ & $290$ & $0.15$ \\
    \textsc{Grokfast-MA} & $790 \, (\times \, 50.49 \downarrow)$ & $292 \, (\times \, 20.49 \downarrow)$ & $458$ & $0.37$ \\
    \textsc{Grokfast} & $910 \, (\times \, 43.84 \downarrow)$ & $137 \, (\times \, 43.79 \downarrow)$ & $294$ & $0.15$ \\
    \bottomrule
    \end{tabular}%
    }
\end{minipage}
\\
\begin{minipage}{\linewidth}
\begin{center}
    \caption{%
    Quantitative results of \textsc{Grokfast} with an MLP trained for MNIST (Figure~\ref{fig:mnist}).
    }
    \label{tab:appx_exp_time:mnist}
    \resizebox{\linewidth}{!}{%
    \begin{tabular}{lcccc}
    \toprule
    Algorithm & Iterations @ 95\% Val. Acc. & Wall Clock Time @ 95\% Val. Acc. (s) & VRAM (MB) & Latency Per Iteration (ms) \\
    \midrule
    Baseline & $44022$ & $1928$ & $196$ & $43.8$ \\
    \textsc{Grokfast} & $2001 \, (\times \, 22.00 \downarrow)$ & $87.8 \, (\times \, 21.96 \downarrow)$ & $198$ & $43.9$ \\
    \bottomrule
    \end{tabular}%
    }
\end{center}
\end{minipage}
\\
\begin{minipage}{0.45\linewidth}
\begin{center}
    \caption{%
    Quantitative results of \textsc{Grokfast} with a G-CNN trained for QM9 (Figure~\ref{fig:qm9}).
    }
    \label{tab:appx_exp_time:qm9}
    \resizebox{\linewidth}{!}{%
    \begin{tabular}{lccc}
    \toprule
    Algorithm & Minimum Val. Loss & VRAM (MB) & Latency Per Iteration (ms) \\
    \midrule
    Baseline & $0.00659$ & $216$ & $40.2$ \\
    \textsc{Grokfast} & $0.00348$ & $216$ & $41.4$ \\
    \bottomrule
    \end{tabular}%
    }
\end{center}
\end{minipage}
\hfill
\begin{minipage}{0.545\linewidth}
\begin{center}
    \caption{%
    Quantitative results of \textsc{Grokfast} with an LSTM trained for IMDb (Figure~\ref{fig:imdb}).
    }
    \label{tab:appx_exp_time:imdb}
    \resizebox{\linewidth}{!}{%
    \begin{tabular}{lcccc}
    \toprule
    Algorithm & Best Val. Acc. & Minimum Val. Loss & VRAM (MB) & Latency Per Iteration (ms) \\
    \midrule
    Baseline & $0.84$ & $0.517$ & $754$ & $20.4$ \\
    \textsc{Grokfast} & $0.90$ & $0.412$ & $762$ & $21.2$ \\
    \bottomrule
    \end{tabular}%
    }
\end{center}
\end{minipage}
% \vskip -0.2in
\end{table}
%%%%%%%

This section delivers additional demonstration of the efficiency of our \textsc{Grokfast} algorithm.
As we have argued in Section~\ref{sec:2_pre:limit} and Section~\ref{sec:3_method}, the additional computational burden from our augmentation is compensated by the larger-scale acceleration of the delayed generalization.
The additional cost of VRAM memory is also negligible compared the baseline.
The time and the memory requirements in the Tables~\ref{tab:appx_exp_time:alg} through~\ref{tab:appx_exp_time:imdb} are measured with a single GTX 1080 Ti GPU.

\section{More Visualization}
\label{sec:e_visual}
%%%%%%%
\begin{figure}[t]
\newcommand{\figwidth}{0.495\linewidth}
\begin{minipage}{0.95\linewidth}
\begin{center}
    \subfloat[\small Parameter trajectories. \label{fig:traj:large}]{\includegraphics[width=\figwidth]{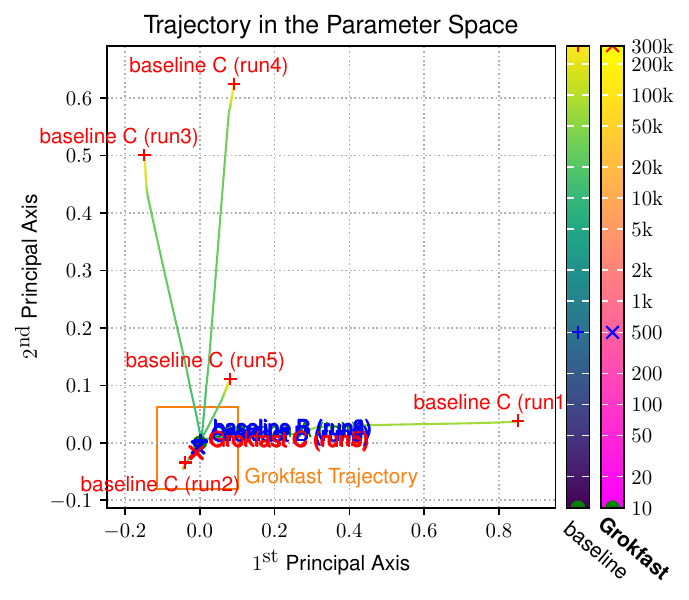}}
    \hfill
    \subfloat[\small Magnification of the {\color[HTML]{f06406} orange} box. \label{fig:traj:mag}]{\includegraphics[width=\figwidth]{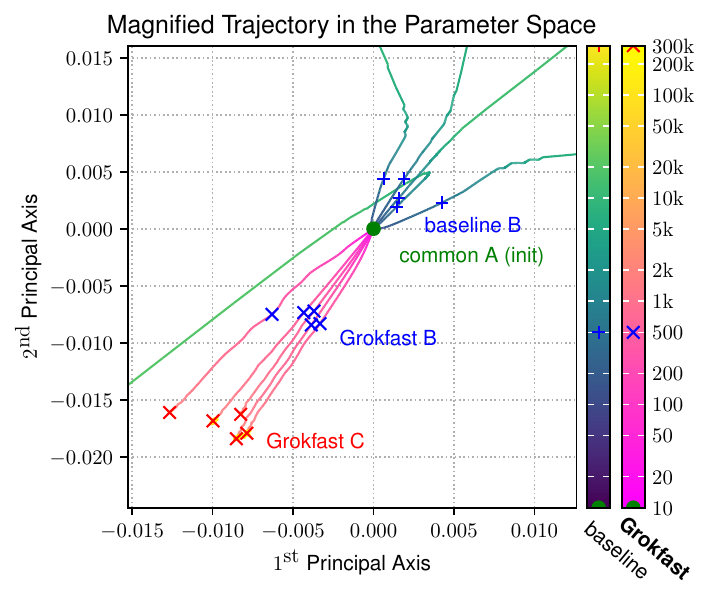}}
    \caption{%
    Normalized trajectories of model parameters from five runs of the experiment of Figure~\ref{fig:figure_one}.
    The baseline optimization algorithm \textit{without} \textsc{Grokfast} guide the model to the overfitting states (state $\textrm{B}$) relatively closer to the initialization states (state $\textrm{A}$).
    After reaching the overfitting state, however, the model parameters travel far away to reach the generalization states (state $\textrm{C}$).
    \textsc{Grokfast} instead guide the model parameters to the alternative generalization states (state $\textrm{C}$), which are much closer to the initialization states (state $\textrm{A}$).
    }
    \label{fig:appx:traj}%
\end{center}
\end{minipage}
\vskip -0.2in
\end{figure}
%%%%%%%

We finally provide more visualization in addition to Section~\ref{sec:5_disc} in order to understand the training dynamics under our \textsc{Grokfast} algorithm.
Figure~\ref{fig:appx:traj} shows five more runs from the same experiments in Figure~\ref{fig:figure_one} and~\ref{fig:traj} with different seeds.
We saved all the 423k parameters of the Transformer decoder~\citep{transformer} from every training iteration, likewise in Figure~\ref{fig:traj}.
The parameters of a model checkpoint from each run at each iteration are reshaped into a single long vector.
The vectorized parameters are then normalized by subtracting them by the model's initialized weights.
This way, we can align the trajectories by centering the initial states (state $\textrm{A}$) of all the experiments at the origin.
The sequence of parameter differences of the ten runs from the two algorithms, i.e., baseline and $\textsc{Grokfast}$, forms a tensor of shape ((Number of Runs) $\cdot$ (Number of Iterations)) $\times$ (Number of Parameters) = $(\textrm{(Number of Sampled Iterations)} \times 422784)\,$.
From this we perform the PCA to obtain the projection matrix of shape $422784 \times 2\,$.
This matrix projects the parameter differences from each of the model checkpoint onto the two most salient directions of variations.
We mark the initialization state (state $\textrm{A}$), the overfitting state (state $\textrm{B}\,$, 500 iterations), and the generalization state (state $\textrm{C}$) from each run in the two-dimensional plot.
The results are Figure~\ref{fig:appx:traj}.

We first notice that the overfitting states (state $\textrm{B}$) from each of the two optimizers are clearly different.
The baseline algorithm without $\textsc{Grokfast}$ reaches the overfitting states (state $\textrm{B}$), which are relatively nearer to the initialization states (state $\textrm{A}$) than those of $\textsc{Grokfast}$ algorithm.
However, as soon as the model overfits, the weights continue to deviate far from the points where overfitting first occured (state $\textrm{B}$).
As a result, the final generalization (state $\textrm{C}$) happens much far away from the initialized weights (state $\textrm{A}$).
It is notable that the five generalization states (state $\textrm{C}$) from different instances of the baseline optimizer vary significantly.
The difference between the states $\textrm{C}$ of the baseline is much larger than that of the baseline's overfitting states (state $\textrm{B}$) and that of the states $\textrm{B}$ and $\textrm{C}$ from our \textsc{Grokfast} algorithm.
In contrast, the training dynamics from \textsc{Grokfast} results in a distinct set of trajectories that lead to the generalization states (state $\textrm{C}$) much closer to the initial weights.
Moreover, difference within the trajectories from \textsc{Grokfast} is much smaller than that of the baseline algorithms.
This conclusion is also verifiable from Figure~\ref{fig:traj:mag} and Table~\ref{tab:traj} in a quantitative manner.

\end{document}